\documentclass[journal]{IEEEtran}

\usepackage{amssymb,amsthm,}
\usepackage{amsmath,amsfonts}
\usepackage{algorithm, algpseudocode,xcolor}
\usepackage{array}
\usepackage[caption=false,font=normalsize,labelfont=sf,textfont=sf]{subfig}
\usepackage{textcomp}
\usepackage{stfloats}
\usepackage{url}
\usepackage{verbatim}
\usepackage{graphicx}
\usepackage{cite}
\usepackage{multicol}
\usepackage{booktabs}
\usepackage{multirow}
\begin{document}
\newtheorem{lemma}{Lemma}
\newtheorem{proposition}{Proposition}
\newtheorem{definition}{Definition}
\newtheorem{theorem}{Theorem}
\newtheorem{assumption}{Assumption}
\title{Learning Aligned Stability in Neural ODEs Reconciling Accuracy with Robustness}

\author{ Chaoyang Luo, Yan Zou, Nanjing Huang{*}
	\thanks{*Corresponding author: Nanjing Huang (e-mail: njhuang@scu.edu.cn; nanjinghuang@hotmail.com).}
	\thanks{Chaoyang Luo and Nanjing Huang are with the Department of Mathematics, Sichuan University, Sichuan, 610064, China (e-mail: luochaoyang@stu.scu.edu.cn).}
	\thanks{Yan Zou is with the Department of Artificial Intelligence and Computer Science, Yibin University, Sichuan, 644000, China (e-mail: zouyan@yibinu.edu.cn) }
	}
\maketitle
\markboth{Luo \MakeLowercase{\textit{et al.}}: Learning Aligned Stability in Neural ODEs Reconciling Accuracy with Robustness}{IEEE TRANSACTIONS ON PATTERN ANALYSIS AND MACHINE INTELLIGENCE,~Vol.~14, No.~8, August~2021}%


\begin{abstract}
Despite Neural Ordinary Differential Equations (Neural ODEs) exhibiting intrinsic robustness, existing methods often impose Lyapunov stability for formal guarantees. However, these methods still face a fundamental accuracy-robustness trade-off, which stems from a core limitation: their applied stability conditions are rigid and inappropriate, creating a mismatch between the model's regions of attraction (RoAs) and its decision boundaries. To resolve this, we propose Zubov-Net, a novel framework that unifies dynamics and decision-making. We first employ learnable Lyapunov functions directly as the multi-class classifier, ensuring the prescribed RoAs (PRoAs, defined by the Lyapunov functions) inherently align with a classification objective.  Then, for aligning prescribed and true regions of attraction (PRoAs-RoAs), we establish a Zubov-driven stability region matching mechanism by reformulating Zubov's equation into a differentiable consistency loss.  Building on this alignment, we introduce a new paradigm for actively controlling the geometry of RoAs by directly optimizing PRoAs to reconcile accuracy and robustness. Our approach uses tripartite losses  (consistency, classification, separation) and a parallel boundary sampling algorithm to co-optimize the Neural ODE and Lyapunov function. To enhance the discrimination of Lyapunov functions, we design a Partially  Input-Attention-based Convex Neural Network via a softmax attention mechanism that focuses on equilibrium-relevant features and serves as weight normalization to maintain training stability in deep architectures. Theoretically, we prove that minimizing the tripartite loss guarantees consistency alignment of PRoAs-RoAs, non-overlapping PRoAs, trajectory stability, and a certified robustness margin. Moreover, we establish stochastic convex separability with tighter probability bounds and lower dimensionality requirements to justify the convex design in Lyapunov functions. Experimentally, on SVHN, CIFAR-10, CIFAR-100, and Tiny-ImageNet, Zubov-Net maintains high clean accuracy while demonstrating superior robustness against diverse stochastic corruptions and both white-box and black-box adversarial attacks. The framework also demonstrates synergy with adversarial training, and its computational efficiency confirms practical feasibility.

\end{abstract}

\section{Introduction} \label{sec:intro}
\IEEEPARstart{R}{ecent} advances in machine learning, particularly deep neural networks (DNNs), have achieved remarkable success in fields such as computer vision \cite{11018619}, natural language processing \cite{10905032}, and robotics \cite{10363393}. However, the reliability of DNNs is challenged by their sensitivity to various input perturbations, including both adversarial attacks (maliciously crafted perturbations) and random noise corruptions, which can induce erroneous predictions and severely hinder deployment in safety-critical applications \cite{10530438,8611298}. This vulnerability raises an urgent need for architectures with enhanced robustness against perturbed inputs.

Neural ODEs, which model continuous-time dynamics, have demonstrated inherent robustness advantages over discrete-depth networks \cite{NEURIPS2018_69386f6b,9157003}. Consequently, significant research has explored their properties for robust learning and defense. For instance, Carrara et al. \cite{9035109} developed error tolerance tuning in ODE solvers; Yan et al. \cite{YAN2020On} proposed TisODE by leveraging non-intersecting integral curves; and Cui et al. \cite{CUI2023576} designed the half-Swish activation function for Neural ODEs.  Despite empirical improvements, these heuristic approaches lack formal robustness guarantees. Lyapunov stability from control theory has emerged as a critical framework for robustness certification.  It works by certifying stabilization conditions to ensure the system converges to a stable equilibrium from any state within a region of attraction (RoA).

Recent works have integrated Lyapunov stability into Neural ODEs in various forms \cite{pmlr-v162-rodriguez22a,huang2022fi,LUO2025107219}. For example, Rodriguez et al. \cite{pmlr-v162-rodriguez22a} proposed a training framework based on control-theoretic Lyapunov stability conditions as loss functions, which demonstrated advantages in convergence speed and robustness. Luo et al. \cite{LUO2025107219} introduced a fixed-time stability condition to ensure convergence within a predefined time. Huang et al. \cite{huang2022fi} enforced forward invariance by imposing stability constraints on sub-level sets of predetermined Lyapunov functions. These methods typically require a fixed predetermined region of attraction and impose stabilization conditions on the entire region to stabilize Neural ODEs. However, this often leads to over- or under-constrained stabilization, exacerbating the tension between accuracy and robustness.

\begin{figure*}[!t] 
	\centering
	\includegraphics[width=\textwidth]{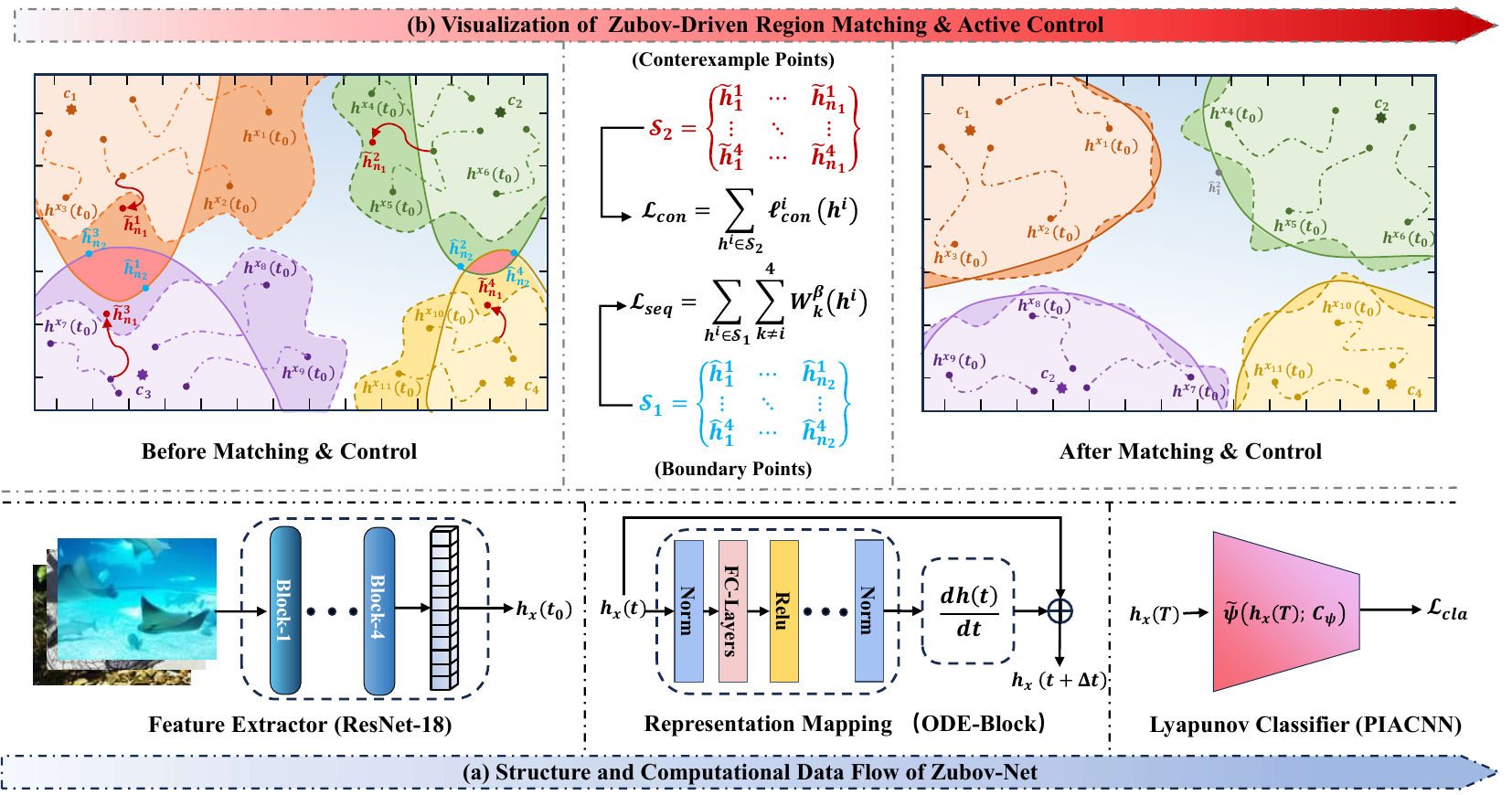}
	\caption{ The overall architecture and working principle of Zubov-Net. (a) Structure and computational data flow of Zubov-Net: ResNet18 serves as the feature extractor, a Neural ODE acts as the representation map, and a novel Lyapunov Classifier implemented by PIACNN. (b) Visualization of Zubov-Driven Region Matching \& Active Control: Solid lines depict the Prescribed Regions of Attraction (PRoAs), as defined by the Lyapunov classifier. Dashed lines represent the true dynamical Regions of Attraction (RoAs) of the Neural ODE. Light/Dark shades indicate aligned/misaligned areas, with light red highlighting conflicting overlaps between adjacent PRoAs. The process from left (``Before Matching \& Control'') to right (``After Matching \& Control'') shows how the Zubov-driven consistency loss (using counterexample points $\mathcal{S}_2$) aligns the PRoAs with the RoAs, while the separation loss (using boundary points $\mathcal{S}_1$) actively shapes their geometry to be well-separated.} 
	\label{fig:fullwidth-top}
\end{figure*}

In contrast, another line of work explicitly enforces stability through a projection-based mechanism. This approach, pioneered by Kolter et al. \cite{NEURIPS2019_0a4bbced}, constructs a learnable convex Lyapunov function and projects the dynamical field to satisfy the Lyapunov condition, thereby enforcing global asymptotic stability. This paradigm has been extended to Deep Equilibrium Models \cite{Chu_Wei_Liu_Zhao_Miyatake_2024} and Physics-Informed Neural Networks \cite{10.24963/ijcai.2024/428} for robust image classification. However, the projection mechanism can limit model expressivity. More critically, its “global projection stability” forces features from different classes to converge near the equilibrium point, inevitably shrinking the inter-class margins and undermining the discriminative power of the classifier \cite{NEURIPS2023_0a443a00}.

Alternatively, data-driven paradigms treat each clean training sample as an individual Lyapunov-stable equilibrium point. For instance, SODEF \cite{NEURIPS2021_7d5430cf} enforces convergence of perturbed inputs to their corresponding clean features, while ASODE \cite{NEURIPS2022_299a08ee} extends this to non-autonomous systems. This equilibrium-centric strategy imposes attraction constraints on every clean data point, overlooking the broader, class-wise attraction basins that should guide perturbed samples to the correct category’s region, not necessarily to each clean sample.

Despite their varied forms, existing methods share a fundamental limitation:  the stability conditions, whether imposed by a fixed PRoA, a global projection, or per-sample as equilibrium points, are rigid and inappropriate. These methods are unable to adapt to the complex and diverse discriminative structure learned from data. This inevitably creates a misalignment between the model’s regions of attraction (RoAs) and its logical decision regions, which is a direct cause of the trade-off between accuracy and robustness.

To overcome this misalignment, we propose Zubov-Net, a novel framework that unifies dynamics and decision-making. First, we employ learnable Lyapunov functions to serve a dual purpose: as a multi-class classifier and as the prescribed region of attraction (PRoA— i.e., the imposed stability region). This design ensures that the PRoA is inherently adaptable and can dynamically align with the learned data distribution. Second, to rigorously align the PRoA with the true region of attraction (RoA), we introduce a Zubov-driven stability region matching mechanism. This mechanism reformulates the core PDE of Zubov’s theorem into a differentiable consistency loss, thereby driving the RoA to align with the data’s discriminative structure.

Crucially, this matching enables a new paradigm: active RoA control. Instead of merely certifying an existing RoA, we actively control its geometry by optimizing the PRoA. In this paper, we optimize the PRoA with a classification loss to guide data-driven trajectories into the correct class's RoA and a separation loss to ensure well-separated RoAs, mitigating the accuracy-robustness tension. While Zubov’s theory has been used to learn Lyapunov functions for RoA approximation or verification \cite{liu2025physics,kang2023data,li2025twostage}, to our knowledge, this is the first work to leverage it for designing a Zubov-driven stability region matching mechanism and establishing an active RoA-control paradigm within Neural ODEs.

Our contributions are summarized as follows:
\begin{itemize}
	\item We propose Zubov-Net, a novel framework that unifies dynamics and decision-making through a learnable Lyapunov classifier. Within this unified architecture, we introduce a Zubov-driven stability region matching mechanism, which in turn enables an active RoA-control paradigm.
	
	\item To realize this unified architecture, we design a Partially Input-Attention-based Convex Neural Network (PIACNN) as the Lyapunov classifier, balancing convexity for stability with enhanced discrimination. To drive the region matching and control, we devise a tripartite loss and a parallel boundary sampling algorithm for efficient co-optimization.	
	
	\item Theoretically, we show that minimizing the tripartite losses ensures PRoAs-RoAs alignment, non-overlapping PRoAs, trajectory stability, and a certified robustness margin. We also provide a stochastic convex separability analysis to justify the convex design of the Lyapunov functions as a classifier.
	
	\item  Extensive experiments, on SVHN, CIFAR-10, CIFAR-100, and Tiny-ImageNet, show that Zubov-Net achieves high accuracy on clean data and significantly improves robustness under various stochastic noises and  white-box/black-box adversarial attacks. The framework's synergy with adversarial training  is further demonstrated.
\end{itemize}

\section{Background and Related Work} \label{sec:2}
This section outlines the relevant background and theoretical foundations, reviews core concepts in Neural ODEs and Lyapunov stability theory (with a focus on Zubov's theorem), and introduces related approaches that integrate stability theory into Neural ODEs for robust learning.

\subsection{Additional Details on Neural ODEs}
\noindent  \textbf{Neural ODEs.} It parameterizes the dynamics of hidden states $h(t) \in H \subset \mathbb{R}^{d_h}$ via a neural network $ f (\cdot;{\theta _f})$. Given a data pair $(x, y) \in \mathbb{D}$, the model is defined as
\begin{align}
	&h_{x}(0)=\phi \left( {x;{\theta _\phi }} \right),  \label{eq:2.1} \\
	&h_{x}(T) = h_x(0)+\int_0^T f (h_x(s);{\theta _f}){\mkern 1mu} ds, \label{eq:2.2}\\
	&\hat{y}_{x} = \psi \left( {h_{x}\left( T \right);{C _\psi }} \right),  \label{eq:2.3}
\end{align}
where \( \phi : \mathbb{R}^{d_x} \to \mathbb{R}^{d_h} \) is the feature extractor and \( \psi : \mathbb{R}^{d_h} \to \mathbb{R}^{d_y} \) is the classifier for image classification with $L$ categories.The system evolves over a time interval $[0, T]$ to produce the final prediction $\hat{y}_x$.

Compared to prior works that employ a linear classifier, Zubov-Net constructs its classifier directly from a learnable Lyapunov function. Indeed, designing an effective multi-class classifier with convex Lyapunov functions presents a significant challenge. While convexity-based approaches have been explored for certified robustness in binary classification \cite{NEURIPS2023_a45b205c}, extending them to a general multi-class setting remains an open problem. Addressing this gap is also a core contribution of our method, as detailed in Section~\ref{sec:3}.

\noindent  \textbf{Model Assumption.} We make the following Lipschitz continuity assumption on neural networks used in Neural ODEs.

\begin{assumption} \label{as:1}
	For any neural network described by \( F( \cdot; \theta_F) :\mathbb{R}^{d_x} \to \mathbb{R}^{d_F} \) with parameter \( \theta_F \), there is a  constant \( L_F > 0 \) such that $
		\left\| F(x; \theta_F) - F(x'; \theta_F) \right\| \leq L_F \left\| x - x' \right\| $ for all \( x, x' \in \mathbb{R}^{d_x} \).
\end{assumption}

This assumption is generally satisfied since neural networks with activation functions (such as tanh, ReLU, and sigmoid functions) generally satisfy the Lipschitz continuity condition \cite{Latorre2020Lipschitz,oh2024stable}. Moreover, we assume that the state space \( H \subset \mathbb{R}^n \) is compact and path-connected, which is crucial for ensuring the existence of a Lyapunov function.

\subsection{Zubov’s Theorem}

For an autonomous nonlinear system in Equation \eqref{eq:2.2}, we set $  \mathcal{A} \subseteq H$ as an invariant set, i.e., \( h_0 \in \mathcal{A} \) implies \( \omega(t, h_0) \in \mathcal{A} \) for all \( t \geq 0 \). Let ${\left\| h \right\|_\mathcal{A}}$ denote the distance from a point $h$ to the set $\mathcal{A}$. Then, the RoA associated with $\mathcal{A}$ is defined as
\begin{equation*}
	\mathcal{D}_{f}(\mathcal{A}) := \left\{ h \in H : \lim_{t \to \infty} {\left\| {{\omega _{{\theta _f}}}(t,h)} \right\|_\mathcal{A}}= 0 \right\}.
\end{equation*}

A continuously differentiable function $\mathcal{V}: H \rightarrow \mathbb{R}$ is called a Lyapunov function for the invariant set $\mathcal{A}$ if it is positive definite, i.e., $\mathcal{V}(h) = 0$ for $h \in \mathcal{A}$ and $\mathcal{V}(h) > 0$ for all $h \notin \mathcal{A}$.

We are interested in utilizing a learnable Lyapunov function to characterize the RoA. This is achieved using sub-level sets (defined as \( \mathcal{V}_c := \{ x \in X : \mathcal{V}(x) < c \} \), where \( c > 0 \) ) to internally approximate the RoA. Therefore, we state Zubov's theorem, which indicates that the RoA is equivalent to the sub-level-1 set of a certain Lyapunov function.

\begin{lemma}[Zubov's Theorem \cite{zubov1961methods}] \label{lem:2.2}
	Let \( W : H \rightarrow [0,1] \) be a continuous function and \( D_W:= \{ h \in H  : W(h) < 1 \} \). Then \( D_W = \mathcal{D}_{f}(\mathcal{A})\) if and only if  the following conditions hold
	\begin{enumerate}
		\item For any sufficiently small \( c_1 > 0 \) and $h \in D$, there exists \( c_2 \in[0,1] \) such that \( {\left\| h \right\|_\mathcal{A}} \geq c_1 \) implies \( W(h) > c_2 \).
		\item \( W(h) \rightarrow 1 \) as \( h \rightarrow \partial D \) and \( W(h) = 0 \) for $h \in \mathcal{A}$.
		\item For any $h \in D$, the derivative of \( W \) along solutions of Equation \eqref{eq:2.2} satisfy
		\begin{equation}  \label{eq:2.4}
			\dot{W}(h) = -\Phi(h)(1 - W(h)),
		\end{equation}
		where \( \Phi : D \rightarrow \mathbb{R} \) continuous and positive definite with respect to \( \mathcal{A}\).
	\end{enumerate}
\end{lemma}

However, it requires designing a learnable function $W(\cdot; \theta_W)$ satisfying conditions 1) and 2), which is a non-trivial task. Interestingly, we can choose a suitable transformation function to construct a function $W(\cdot; \theta_W)$ from a candidate Lyapunov function $\mathcal{V}$ with $\mathcal{V}(h)=0$ iff $h \in \mathcal{A}$ as follows
\begin{equation} \label{eq:2.5}
	W\left( x \right) = 1 - \exp \left( { - \mathcal{V}\left( x \right)} \right),
\end{equation}
and more technical details can be found in \cite{liu2025physics,kang2023data}. Further, we can set $D_{W}$ (the sub-level-1 set of $W(\cdot;\theta_{W})$) as the PRoA.  

Moving beyond its traditional role as a static verification tool, we leverage Zubov's equation (Eq. \eqref{eq:2.4}) as a differentiable consistency characterization between the PRoA (defined by the Lyapunov classifier) and the true dynamical RoA of the Neural ODE. By enforcing this consistency during training, we establish a Zubov-driven stability region matching mechanism.

\subsection{Learning Stable Dynamics}
Research on stable Neural ODEs has increasingly incorporated stability theory, leading to methods that induce robustness by modifying the underlying dynamics. For instance, some approaches enforce stability through architectural biases like skew-symmetric structures \cite{pmlr-v145-huang22a} or contraction theory \cite{9809979}, while others apply post-hoc stabilization via minimal weight perturbations \cite{de2025improving}. Extensions to graph networks leverage fractional-order dynamics for tighter output bounds \cite{Kang_Zhao_Song_Xie_Zhao_Wang_She_Tay_2024,Cui_Kang_Li_Zhao_Tay_Deng_Li_2025} .  In these works, the stability mechanisms are treated as a separate, implicit property of the differential equation, often decoupled from the classification objective.  Zubov-Net, in contrast, introduces a paradigm that unifies dynamics and decision-making. It directly employs a learnable Lyapunov function as the classifier and actively aligns the prescribed stable regions (PRoAs) with the true dynamical regions of attraction (RoAs). This explicit, architecture-level unification ensures that the model's stability is inherently coupled with its classification objective.

Notably, Zhao et al. \cite{NEURIPS2023_0a443a00} critically demonstrate that naive Lyapunov stability is insufficient for multi-class robustness. They propose a Hamiltonian-based framework that implicitly defines attractors via a learned energy function.  In contrast, Zubov-Net introduces a principled, explicit alignment mechanism. This direct alignment offers a more fundamental resolution to the trade-off, moving beyond implicit stability guarantees toward actively controlled, decision-aware stable geometry.

\section{Methodology}\label{sec:3}

This section details Zubov-Net, a framework designed to address the stability-decision misalignment in Neural ODEs. As illustrated in Figure \ref{fig:fullwidth-top}, the core of our approach consists of three synergistic components: a unified architecture, a Zubov-driven region matching mechanism, and an active region of attraction (RoA) control paradigm. We first introduce the Partially Input-Attention-based Convex Neural Network (PIACNN), which instantiates the unified architecture by constructing a multi-class Lyapunov classifier. We then formulate a tripartite training objective to implement the matching and control mechanisms: (i) a consistency loss $\mathcal{L}_{con}$ that enforces the Zubov-driven alignment between prescribed and dynamical attraction regions; (ii) a classification loss $\mathcal{L}_{cla}$ that certifies trajectory convergence to the correct class basin; and (iii) a separation loss $\mathcal{L}_{sep}$ that actively shapes the basin geometry to maximize inter-class margins. Finally, we present the end-to-end training procedure, where these components are jointly optimized, as summarized in Algorithm \ref{alg:2}.

\subsection{The Framework of Zubov-Net}\label{sec:merhods_A}
Prior approaches to stabilizing Neural ODEs rely on a split architecture: stability constraints are applied to the feature dynamics, while final classification is performed by a separate linear layer.  This decouples the logical decision region from the underlying regions of attraction (RoA), creating a structural misalignment that inherently worsens the accuracy-robustness trade-off.

To resolve this, Zubov-Net introduces a unified architecture where learnable Lyapunov functions serve directly as the multi-class classifier, merging dynamics and decision-making. For $L$ classes, let $\mathcal{A} = {c_1, \dots, c_L}$ be the invariant set of class-specific equilibrium points. For each $c_i$, we construct a Lyapunov function $W_i(\cdot; \theta_W)$ from a learnable positive definite function $\mathcal{V}_i(\cdot; \theta_\mathcal{V})$ via Equation \eqref{eq:2.5}.

Implementing this unified architecture requires a classifier that is both a valid positive definite function  $\mathcal{V}$ (convex for positive definite) and sufficiently discriminative for multi-class problems. While Input Convex Neural Networks (ICNNs) are common for learning positive definite functions \cite{NEURIPS2019_0a4bbced, Chu_Wei_Liu_Zhao_Miyatake_2024}, they scale poorly to multi-class settings as they require $L$ independent networks. We therefore design a Partially Input-Attention-based Convex Neural Network (PIACNN), building upon the PICNN \cite{pmlr-v70-amos17b} and incorporating a novel input-attention mechanism. This mechanism enhances discrimination while preserving convexity about the input, and also acts as a weight normalization to prevent output explosion and stabilize deep architectures.

\begin{figure}[!t]
	\centering
	\includegraphics[width=\columnwidth]{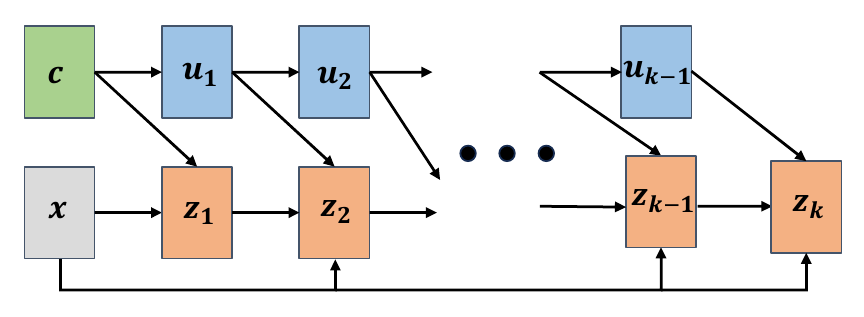}
	\vspace{-20pt}
	\caption{A partially input-attention-based convex neural network (PIACNN)}
	\vspace{-15pt}
	\label{fig:IACNN}
\end{figure}

Let $g(\cdot)$ denote a PIACNN whose architecture is detailed in Figure \ref{fig:IACNN}, and define the positive definite function for an equilibrium $c$ as
\begin{equation*} 
	\mathcal{V}(x,c;\theta_{\mathcal{V}}) = \sigma (g(x - c,c) - g(\mathbf{0},c)) + \delta \left\| {x - c} \right\|_2^2,
\end{equation*}
where $\sigma(\cdot)$ is a positive convex non-decreasing function with $\sigma (0) = 0$, and $\delta > 0$ is a small constant. The forward pass of the $k$-layer PIACNN is:
\begin{align}
	u_{i+1} &= \tilde{\sigma}_i(\tilde{A}_i u_i + \tilde{b}_i),  \label{eq:3.1}\\
	z_{i+1} &= \sigma_i \left( A_i^{(z)} \left( \rm{softmax}\left( A_i^{(zu)} u_i + b_i^{(z)} \right) \circ z_i \right) \right. \nonumber \\
	&\quad + \left. A_i^{(x)} \left( \rm{softmax}\left( A_i^{(xu)} u_i + b_i^{(x)} \right) \circ x \right) \right. \nonumber \\
	&\quad + \left. A_i^{(u)} u_i + b_i \right) \label{eq:3.2} , \\
	g(x, c) &= z_k, \quad u_0 = c , \quad z_0=x \label{eq:3.3},
\end{align}
where $u_i, z_i$ are hidden units, $\circ$ is the Hadamard product, and convexity in $x$ is preserved by enforcing $A_i^{(z)} \geq 0$ and using convex non-decreasing activations $\sigma_i(\cdot)$. The following proposition formalizes the convexity guarantee of PIACNN, with proof in the Appendix \ref{app:pro1}.
\begin{proposition} \label{propos:1}
	Let $A_i^{(z)}$ ($i = 0, \dots, k-1$) be nonnegative matrices and $\sigma_i(\cdot)$ convex non-decreasing activation functions. Then the $k$-layer PIACNN $g(x, c)$ defined in \eqref{eq:3.1}-\eqref{eq:3.3} is convex in $x$.
\end{proposition}

A natural concern is whether a classifier based on convex functions possesses sufficient expressiveness. The following proposition provides a probabilistic guarantee, justifying our design choice theoretically. The proof is in the Appendix \ref{app:pro2}. 
\begin{proposition}\label{pro:convex-sep}
	Let $X_1 = \{x^{(i)}\}_{i=1}^N \subset [-1,1]^d$ and $X_2 = \{y^{(j)}\}_{j=1}^M \subset [-1,1]^d$ be i.i.d. uniform samples with $M \ge N$. The probability of convex separability satisfies:
	\begin{equation}
		\mathbb{P}\left(X_2 \cap \operatorname{co}(X_1) = \emptyset \right) \geq 
		\begin{cases} 
			{\left( {1 - {{\left( {\frac{{N - 1}}{{N + 1}}} \right)}^d}} \right)^M}& \forall d \in \mathbb{N}; \\
			1 & d \ge N+1. 
		\end{cases} \nonumber
	\end{equation}
	where $\operatorname{co}(\cdot)$ denotes the convex hull operation.
\end{proposition}
While derived under a uniform distribution assumption, Proposition \ref{pro:convex-sep} offers a fundamental insight: high dimensionality facilitates convex separability. Our proof, inspired by \cite{NEURIPS2023_a45b205c}, employs a pointwise separation strategy to relax dimensionality requirements and achieve a tighter bound than standard results.

\subsection{Formulation of Training Objective}\label{sec:FTO}
As established in Subsection \ref{sec:merhods_A}, Zubov-Net introduces a unified architecture that directly employs the Lyapunov function as a classifier. Building upon this, to implement a Zubov-driven region matching mechanism and an active RoA control paradigm, we design a tripartite training objective that jointly realizes the following:
\begin{algorithm}[t]
	\caption{Parallel  Boundary Sampling}\label{alg:parallel}
	\begin{algorithmic}[1]
		\State \textbf{Input:} Lyapunov function $W$, equilibrium $c \in \mathbb{R}^d$, target value $\rho$, tolerance $\epsilon > 0$, max iterations $n$, and search directions $\hat{Q} = \{q_i \in \mathbb{R}^d \mid i=1,\dots,d^h\}$.
		\State \textbf{Initialize:}
		\Statex \quad $s \gets \mathbf{0}_{|\hat{Q}|}$ \Comment{Cumulative length vector}
		\Statex \quad $a \gets \mathbf{1}_{|\hat{Q}|}$ \Comment{Adaptive step vector}
		\Statex \quad $\mathcal{B} \gets \{1,\dots,|\hat{Q}|\}$ \Comment{Active index set}
		
		\For{iteration $k=1$ to $n$}
		\If{$\mathcal{B} = \emptyset$}
		\State \textbf{break} \Comment{All directions converged}
		\EndIf
		
		\State $s[\mathcal{B}] \gets s[\mathcal{B}] + a[\mathcal{B}]$ \Comment{Parallel length update}
		\State $P \gets c +   (s[\mathcal{B}] \circ \hat{Q}[\mathcal{B}])$ \Comment{Batch point calculation}
		\State $V \gets W(P)$ \Comment{Parallel function evaluation}
		
		\State $M \gets \mathbf{0}_{|\mathcal{B}|}$ \Comment{Direction mask}
		\State $M[V < \rho - \epsilon] \gets 1$ \Comment{Outward move needed}
		\State $M[V > \rho + \epsilon] \gets -1$ \Comment{Inward move needed}
		
		\State $\Lambda \gets \frac{|M + \operatorname{sign}(a[\mathcal{B}])| + 2}{4}$ \Comment{Unified scaling factor}
		\State $a[\mathcal{B}] \gets |a[\mathcal{B}]| \circ \Lambda \circ M$ \Comment{Parallel step update}
		
		\State $\mathcal{B} \gets \{ i \in \mathcal{B} \mid M_i \neq 0 \}$ \Comment{Prune converged directions}
		\EndFor
		
		\State \textbf{Output:} $S \gets \{ c +   s_i \cdot \hat{q}_i \mid i=1,\dots,|\hat{Q}|\}$
	\end{algorithmic}
\end{algorithm}

\begin{enumerate}
	\item \textbf{PRoAs-RoAs Consistency} via a novel consistency loss $\mathcal{L}_{con}$ derived from Zubov's equation (Eq. \eqref{eq:2.4}). This component quantifies and minimizes the PRoAs-RoAs discrepancy, ensuring $\mathcal{D}_{f}$ is characterized by $D_{W}$.
	
	\item \textbf{Certified Trajectory Asymptotic Stability} through a classification loss $\mathcal{L}_{cla}$ designed specifically for the Lyapunov classifier, guarantees that the evolutionary trajectory $h_x(t)$ converges to and remains within its correct class's PRoA.
	
	\item \textbf{Geometric Basin Separation} by a boundary separation loss $\mathcal{L}_{sep}$ that actively optimizes PRoA geometry to maximize inter-class margins, implementing the active control paradigm.
	
\end{enumerate}
These losses are optimized jointly: $\mathcal{L}_{con}$ aligns dynamics with decision regions, $\mathcal{L}_{cla}$ operates within these aligned regions, and $\mathcal{L}_{sep}$ refines their boundaries for robustness. We now detail the formulation of each loss function.

\vspace{\baselineskip} 
\noindent  \textbf{Consistency Losses.} 
For any $c_i \in \mathcal{A}$, let $W_i(h) = W(h, c_i; \theta_{W})$ be a Lyapunov function and $\Phi_i$ be a continuous positive definite function. For any state $h$ within $D_{W_i}$, we define the pointwise consistency loss as:
\begin{equation*}
	l_{con}^i(h) = \left\| \frac{dW_i(h)}{dh} f(h;\theta_f) + \Phi_i(h)(1 - W_i(h)) \right\|_2^2.
\end{equation*}

This point-wise definition quantifies the violation of consistency at local state $h$. For the global consistency, we formulate the supremum-based loss as follows
\begin{equation}\label{eq:3.4}
	{\mathcal{L}_{con}}\left( \theta_f,\theta_W  \right) = \sum\limits_{i = 1}^L {\mathop {\sup }\limits_{h \in {D_{W_i}}} l_{con}^i(h)},
\end{equation}
where $L$ is the number of classes. Minimizing $\mathcal{L}_{con}$ directly enforces the Zubov-driven matching mechanism, providing the following guarantees (see the Appendix \ref{app:pro3} for proofs).

\begin{proposition} \label{propos:2}
	Consider the consistency loss $\mathcal{L}_{con}$ defined in \eqref{eq:3.4}. The equality $\mathcal{D}_{f}(\mathcal{A}) = \bigcup_{c_i \in \mathcal{A}} D_{W_i}$ holds if and only if there exists $\theta^*$ such that $\mathcal{L}_{con}(\theta^*) = 0$.
\end{proposition}

Moreover, we will further demonstrate that the alignment inherently yields non-overlapping attraction basins because of the mutual exclusivity among different asymptotic equilibrium points, as shown in Proposition \ref{propos:3} (the proof in the Appendix \ref{app:pro4}).

\begin{proposition} \label{propos:3}
	Consider the consistency loss $\mathcal{L}_{con}$ defined in \eqref{eq:3.4}. If $\theta^* \in \Theta$ satisfies $\mathcal{L}_{con}(\theta^*) = 0$, then for any $i \neq j$, the PRoAs satisfy $D_{W_i} \cap D_{W_j} = \emptyset$.
\end{proposition}

\vspace{\baselineskip} 
\noindent  \textbf{Classification Loss.} 
While Proposition \ref{propos:2} and \ref{propos:3} demonstrate consistency and non-overlapping among PRoAs, they do not guarantee that the trajectory converges to the correct equilibrium. This is enforced by the classification loss, designed for the Lyapunov classifier. 

Given $(x,y) \in \mathbb{D}$ as the input-output pair, let $h_x( \cdot)$ be the solution of the system of Equations \eqref{eq:2.1}-\eqref{eq:2.3}, then the Lyapunov value vector and classification loss are defined as
\begin{align*}
	& \mathbf{W}(h_x(T)) = \left[ W_1(h_x(T)), \dots, W_L(h_x(T)) \right]^\top ,\\
	& \tilde \psi \left( {{h_x}\left( T \right);{C_\psi }} \right) = \frac{ \mathbf{W}(h_x(T))^{\circ -1} - \alpha \mathbf{1} }{ \mathbf{1}^\top \left( \mathbf{W}(h_x(T))^{\circ -1} - \alpha \mathbf{1} \right) } ,\\
	& \mathcal{L}_{cla}(\hat{\mathbf{y}}_x, y) = -\log \left( \hat{\mathbf{y}}_x[y] \right), \quad  \hat{\mathbf{y}}_x = \tilde  \psi \left( {h_{x}\left( T \right);{C _\psi }} \right),
\end{align*}
where $\alpha \in [0,1)$ is a stability margin, $\mathbf{1}$ is an all-ones vector, and $^{\circ -1}$ denotes element-wise inverse. This formulation encourages $W_y(h_x(T)) \to 0$ and $W_{j\neq y}(h_x(T)) \to 1$, ensuring $h_x(T) \in D_{W_y}$. 

The combined effect of $\mathcal{L}_{con}$ and $\mathcal{L}_{cla}$ demonstrates full trajectory containment, as shown in the following proposition (the proof in the Appendix \ref{app:pro5}).
\begin{algorithm*}[t]
	\caption{The Complete Training Algorithm} \label{alg:2}
	\begin{algorithmic}[1]
		\State \textbf{Input:} Initial parameters $\theta$, learning rate $\eta_{1}$, number of iterations $N_{1}$, gradient projection stepsize $\eta_{2}$, number of gradient projection iterations $N_{2}$, number of boundary samples $n$,  time discretization resolution $\Gamma$, and  integral trajectory sampling spaced $t_0, t_1, \ldots, t_\Gamma$.
		\For{iteration $1$ to $N_{1}$}
		\State $(x, y) \sim D$  \Comment{Sample training data}
		\State  $\mathcal{S}_1 = \{h_i:h_i\in \partial D_{W_y}, i=1,\cdots,n \}$ \Comment{Calculate the boundary points of $D_{W_y}$ by using Algorithm \ref{alg:parallel} } 
		\State $h_x(0)=\phi \left( {x;{\theta _\phi }} \right)$  \Comment{Compute the extracted features of the input $x$}
		\State $h(t_i) \leftarrow \int_{t_{i-1}}^{t_i} f(h(\tau)) \, d\tau + h(t_{i-1}) \quad \forall \, 1 \le i \le \Gamma$ \Comment{Compute the integral trajectory in the input space}
		\State $\mathcal{S}_2 = \{h_x(t_i), i=1,\dots,\Gamma \}$
		\For{iteration $i=0$ to $\Gamma$}
		\State $\hat{h}(i)=h(t_{i})$ \Comment{Assign initial value as $h(t_{\Gamma})$}
		\For{iteration $1$ to $N_{2}$}
		\State ${\hat{h}(i)} \leftarrow {\hat{h}(i)} + {Proj}_{D_{W_y}}( \eta_{2}{\nabla} l_{con}^y(\hat{h}(i))  )$ \Comment{Calculate the projection gradient and update $\hat{h}(i)$}
		\EndFor
		\EndFor
		\State $\mathcal{S}_2=\mathcal{S}_2 \cup \{\hat{h}(t_i):i=1,\cdots,\Gamma \}$ 
		\State $\mathcal{L} = {\mathcal{L}_{cla}}(\tilde  \psi \left( {h_{x}\left( T \right);{C _\psi }} \right),y) + {\lambda _1}{\mathcal{L}_{FC}}( \psi \left( {h_{x}\left( T \right);{C _\psi }} \right),y) + \sum\limits_{h \in {{\mathcal{S}}_{\rm{2}}}} {{\lambda _2}l_{con}^y(h)} + \sum\limits_{h \in {\mathcal{S}_1}} {\sum\limits_{k \ne y}^L -\lambda_3{W_k^\beta(h)} } $ \Comment{Calculate total loss} 
		\State $\theta  \leftarrow \theta  - {\eta _1}{\nabla _\theta }\mathcal{L}$ \Comment{Compute gradient to update $\theta$} 
		\EndFor
		\State \textbf{Output:} Trained parameters $\theta$.
	\end{algorithmic}
	
\end{algorithm*}

\begin{proposition} \label{propos:4}
	Consider the consistency loss $\mathcal{L}_{con}$ defined in \eqref{eq:3.4}, an input-output pair $(x,y)$, and parameters $\theta^* \in \Theta$ satisfying $\mathcal{L}_{con}(\theta^*) = 0$. If the terminal hidden state $h_x(T) \in D_{W_y}$, then the entire trajectory $\{h_x(t) : t \in [0,T]\}$ is contained in $D_{W_y}$. 
\end{proposition}

The proposition shows that the entire trajectory of a correctly classified sample remains within its correct class's region of attraction, guaranteeing asymptotic stability toward the equilibrium $c_y$. In prior stable Neural ODE works \cite{NEURIPS2021_7d5430cf}, such stability provides a robustness analysis: a sufficiently small perturbation around the clean feature will not cause the terminal state to escape the correct basin, and thus the prediction remains unchanged. We further derive a Lipschitz-based robustness margin under input perturbations as the following proposition, which is proved in the Appendix \ref{app:pro6}.

\begin{proposition}	\label{prop:robustness}
Assume the feature extractor $\phi$ and Lyapunov function $W_y$ satisfy Lipschitz constants $L_\phi$ and  $L_W$. Consider a correctly classified clean sample $(x,y)$ such that $h_x(0)=\phi(x)$, $h_x(T)\in D_{W_y}$, and parameters $\theta^*\in\Theta$ satisfy $\mathcal{L}_{con}(\theta^*)=0$. Define the certified robustness radius $\varepsilon^*$ as:
\begin{equation}\label{eq:certified_radius}
\varepsilon^* = \frac{1 - W_y(h_x(0))}{L_\phi L_W} >0.
\end{equation}
Then, for any perturbed input $x' = x + \delta$ with $\|\delta\| < \varepsilon^*$, we have $h_{x'}(T) \in D_{W_y}$.
\end{proposition}

Proposition \ref{prop:robustness} provides a direct theoretical link between the geometry of the learned stable region and robustness against norm-bounded input perturbations. It offers a concrete margin $\varepsilon^*$ for a correctly classified sample, derived from the Lipschitz continuity of model components, a common practice in the stability-guaranteed Neural ODE literature \cite{pmlr-v162-rodriguez22a, LUO2025107219}. 

\vspace{\baselineskip}
\noindent\textbf{Boundary Separation Loss.}
To actively shape robust basins beyond mere non-overlap (Proposition \ref{propos:3}), we introduce a separation loss. Let $W_k^\beta(h) = 1 - \exp(-\beta \mathcal{V}_k(h))$ with $\beta \in (0,1)$. For boundary points $\{ h_j^{(i)} \}_{j=1}^m \in \partial D_{W_i}$ sampled via Algorithm \ref{alg:parallel}, the loss is:
\begin{equation}\label{eq:3.5}
	\mathcal{L}_{sep}(\theta_W) = \sum_{i=1}^L \sum_{j=1}^m \sum_{k \neq i}^L -W_k^\beta(h_j^{(i)}).
\end{equation}
Minimizing $\mathcal{L}_{sep}$ effectively \textit{maximizes} $\mathcal{V}_k$ at the boundary of $D_{W_i}$, repelling it from the foreign equilibrium point $c_k$. Strong convexity ensures this separation propagates inward.
\begin{proposition} \label{propos:5}
	For strongly convex $\mathcal{V}_k$ ($k = 1,\dots,L$) with $c_k \notin D_{W_i}$ ($k \neq i$), any point $h \in D_{W_i}$ satisfies $\mathcal{V}_k(h) > \inf_{\tilde{h} \in \partial D_{W_i}} \mathcal{V}_k(\tilde{h})$.
\end{proposition}
\begin{proof}
	By strong convexity, $\mathcal{V}_k$ decreases monotonically along rays toward $c_k$. Since $c_k \notin D_{W_i}$, the segment $[c_k, h]$ intersects $\partial D_{W_i}$ at $\tilde{h}$. Radial monotonicity implies $\mathcal{V}_k(h) \geq \mathcal{V}_k(\tilde{h})$, with strict inequality when $h \neq \tilde{h}$.
\end{proof}

Algorithm \ref{alg:parallel} efficiently samples boundary points by leveraging radial monotonicity for binary search and employing parallel tensor operations while replacing logical conditional branching. The complete details and the computational complexity analysis are provided in the Appendix \ref{app:Alg2}.

\subsection{Training Algorithms } \label{sec:learning_algorithms}

To jointly realize the region matching and active control, given a dataset $\mathbb{D}$, we formulate it as the following optimal control problem
\begin{align}
	\mathop{\arg \min}_{\theta \in \Theta} \quad  &\mathbb{E}_{(x,y) \sim \mathbb{D}}    \left[ \mathcal{L}_{cla}(\hat{\mathbf{y}}_x, y) + \lambda_1 \mathcal{L}_{FC}(\hat{y}_x, y) \right]  \nonumber\\
	&+ \lambda_2 \mathcal{L}_{con}(\theta_f,\theta_W) + \lambda_3\mathcal{L}_{sep}(\theta_W) \label{eq:oq1} \\
	\text{s.t. }  &h_{x}(0)=\phi \left( {x;{\theta _\phi }} \right),  \nonumber \\
	&h_{x}(T) = h_x(0)+\int_0^T f (h_x(s);{\theta _f}){\mkern 1mu} ds, \nonumber\\
	&\hat{y}_{x} = \psi \left( {h_{x}\left( T \right);{C _\psi }} \right),  \quad  \hat{\mathbf{y}}_x = \tilde  \psi \left( {h_{x}\left( T \right);{C _\psi }} \right), \nonumber
\end{align}
where $\lambda_{1,2,3}$ are loss weights. The core components are the tripartite losses: $\mathcal{L}_{con}$ drives the Zubov-driven region matching mechanism, $\mathcal{L}_{cla}$ ensures that trajectories $h_x(t)$ are within its correct class's PRoA, and $\mathcal{L}_{sep}$ implements the active geometric control of the RoAs. 

The term $\mathcal{L}_{FC}$ is a standard cross-entropy loss, applied as an auxiliary linear classification layer $\psi(h) = C_\psi h$. This auxiliary loss provides stable supervision, particularly during early training phases when the Lyapunov classifier is not yet discriminative. The final prediction is given solely by the Lyapunov classifier $\tilde{\psi}$. The rows of $C_\psi$  are also the class-specific equilibrium points $\{c_i\}$, and are determined by minimizing the maximum cosine similarity between them to encourage separated initial conditions for different class basins \cite{NEURIPS2021_7d5430cf}.

For optimizing Equation \eqref{eq:oq1}, we compute the dynamics via solver backpropagation through the ODE solver or the adjoint method \cite{NEURIPS2018_69386f6b,kidger2021hey,10045756}, with implementation details provided in Algorithm \ref{alg:2}. Unlike prior stability-enforcing methods that rely on uniform sampling over the entire PRoA \cite{huang2022fi,pmlr-v162-rodriguez22a}, our approach efficiently generates counterexamples using computationally cheaper projected gradient descent (PGD). The results demonstrate that these counterexamples sufficiently guide the training process,  which is consistent with similar observations in the machine learning community \cite{balunovic2020adversarial,de2022ibp}.

\begin{table*}[ht]
	\centering
	\caption{ Classification accuracy and robustness against stochastic noises. ``Original" denotes the classification accuracy on original data. ``Avg" denotes the average accuracy of the eight random noise-perturbed images over the four datasets. The best result is in bold. Unit: \%.} 
	\label{table:acc_noise}
	\begin{tabular}{cccccccccccc}
		\toprule
		Dataset                  & Method     & Original        & Gaussian       & Glass          & Shot           & Impulse        & Speckle        & Motion         & Brightness     & Contrast       & Avg  \\ \midrule
		
		\multirow{6}{*}{SVHN} & ResNet18        & 97.03            & 84.37             & 92.47          & 80.66          & 75.23            & 91.80            & 90.37           & 69.18               & 81.08             & 83.14  \\
		& Neural ODE      & 97.11           & 84.55             & 92.61          & 80.41          & 75.37            & 91.93            & 90.73           & 69.30               & 81.51             & 83.30          \\
		& Proj-NODE	      & 97.09           & 85.49 			& 92.71 	 		& 81.83 	     & 80.95 	      & 91.40 	         & 91.00 	        & 71.52 	      & 85.87 	            & 85.10   \\
		& SODEF           & 97.17            & 86.42             & 92.84          & 83.51          & 82.13            & 92.46            & 91.07           & \textbf{73.33}      & 86.27             & 86.00          \\
		& FxTS-Net        & 96.92            & 86.39             & 92.87          & 82.28          & 81.30            & 92.21            & 90.96           & 71.62               & 84.28             & 85.24         \\
		& LyaNet          & 96.76            & 86.18             & 92.74          & 82.63          & 80.37            & 92.07            & 90.60           & 71.72               & 83.88             & 85.02          \\
		& Zubov-Net       & \textbf{97.36}   & \textbf{86.90}    & \textbf{93.24} & \textbf{83.95} & \textbf{82.76}   & \textbf{92.50}   & \textbf{91.36}  & 72.68               & \textbf{86.55}    & \textbf{86.24} \\ \midrule
		\multirow{6}{*}{CIFAR-10} & ResNet18        & 91.11            & 68.28             & 68.36          & 62.84          & 76.77            & 75.41            & 77.63           & 82.08               & 78.01             & 73.67        \\
		& Neural ODE      & \textbf{91.45}            & 68.62             & 71.12          & 62.46          & 77.15            & 75.96            & 78.82           & 81.60               & 79.32             & 74.38        \\
		& Proj-NODE       &90.03	        & 68.69				& 71.51			 & 62.87		  & 80.34			 & 77.40				& 80.18			  & 85.38				& 83.29				& 76.21       \\
		& SODEF           & 90.19            & 69.20             & 73.65          & 63.88          & 81.14            & 77.85            & 82.61           & \textbf{86.14}      & \textbf{85.35}    & 77.48        \\
		& FxTS-Net        & 88.52            & 68.78             & 72.76          & 64.02          & 80.24            & 77.74            & 80.73           & 85.19               & 83.79             & 76.66        \\
		& LyaNet          & 88.34            & 67.35             & 71.35          & 63.38          & 80.11            & 76.51            & 79.36           & 85.07               & 83.53             & 75.83        \\
		& Zubov-Net       & 91.29   		& \textbf{71.06}    & \textbf{76.43} & \textbf{66.71} & \textbf{81.90}   & \textbf{78.80}   & \textbf{82.86}  & 86.00               & 84.92            &  \textbf{78.59}       \\ \midrule
		\multirow{6}{*}{CIFAR-100}  & ResNet18        & 70.44            & 23.45             & 20.18          & 18.05          & 23.73            & 33.59            & 49.39           & 37.62               & 18.98             & 28.12     \\
		& Neural ODE      & 70.81            & 23.65             & 21.26          & 18.58          & 23.92            & 34.49            & 49.21           & 38.20               & 19.61             & 28.62        \\
		& Proj-NODE	      & 69.55 	         & 27.94 	         & 24.37 	      & 22.81 	       & 28.94 	          & 40.02 	         & 51.43  	       & 34.59 	             & 19.92 	         & 31.25 \\	
		& SODEF           & 70.38            & 28.75             & 23.80          & 22.38          & 28.63            & 38.66            & 51.60           & 39.80               & \textbf{22.98}    & 32.08        \\
		& FxTS-Net        & 70.20            & 26.49             & 24.42          & 21.50          & 25.93            & 37.48            & 50.01           & \textbf{40.23}      & 20.15             & 30.78        \\
		& LyaNet          & 70.24            & 25.60             & 22.53          & 20.01          & 24.50            & 36.22            & 49.61           & 39.76               & 19.87             & 29.76        \\
		& Zubov-Net       & \textbf{71.54}   & \textbf{32.11}    & \textbf{26.22} & \textbf{24.98} & \textbf{31.03}   & \textbf{42.38}   & \textbf{52.37}  & 39.85               & 20.30             & \textbf{33.66}        \\  \midrule
		\multirow{6}{*}{Tiny-ImageNet}  & ResNet18   & 51.37          & 38.18          & 8.61           & 36.01          & 42.34          & 40.93          & 22.25          & 26.43          & 19.84          & 29.32          \\
		& Neural ODE & 52.15          & 38.26          & 9.74           & 36.29          & 42.74          & 41.44          & 22.84          & 27.03          & 20.43          & 29.85          \\
		& Proj-NODE  & 51.94          & 39.03          & 8.96           & 39.63          & 46.29          & 45.15          & 26.13          & 29.16          & 24.53          & 32.36          \\
		& SODEF      & 52.95          & 40.11          & 8.85           & \textbf{40.20} & 46.88          & 45.28          & 25.59          & \textbf{29.68} & 26.12          & 32.84          \\
		& FxTS-Net   & 49.44          & 36.78          & 8.15           & 38.41          & 46.12          & 44.22          & 25.07          & 28.20          & 20.69          & 30.96          \\
		& LyaNet     & 49.35          & 37.38          & 8.14           & 37.93          & 45.52          & 44.09          & 25.40          & 28.48          & 20.38          & 30.92          \\
		& Zubov-Net  & \textbf{53.65} & \textbf{40.38} & \textbf{10.14} & 39.78          & \textbf{47.33} & \textbf{45.71} & \textbf{26.33} & 29.33          & \textbf{26.84} & \textbf{33.23}\\
		\bottomrule
	\end{tabular}
\end{table*}

\section{Experiments}
This section evaluates the performance of Zubov-Net, with a primary focus on comparisons within stable Neural ODE frameworks. The following research questions are addressed:
\begin{itemize}
	\item \textbf{RQ1 (Accuracy \& Noise Robustness):} Does Zubov-Net maintain high clean accuracy while improving robustness against stochastic noise, compared to stable Neural ODE baselines? Is this advantage preserved on larger-scale datasets?
	
	\item \textbf{RQ2 (Adversarial Robustness):} Does Zubov-Net achieve improved robustness against diverse white-box attacks, even including stronger adaptive attacks and black-box attacks? How does it compare with mainstream adversarial training methods?

	\item \textbf{RQ3 (Geometric Evidence):} How does the proposed method shape the feature space geometric structure to enhance robustness under clean and perturbed data?
	
	\item \textbf{RQ4 (Component \& Design Analysis):} What are the individual contributions of the tripartite losses, the Lyapunov classifier architecture, and other design choices (e.g., ODE solver, integration time) to the overall performance?
\end{itemize}

This structured evaluation establishes a clear baseline for the proposed contributions. The compatibility and synergy between Zubov-Net and other defense paradigms are also explored.

\subsection{Experimental Settings}
This subsection elaborates on the experimental setup, encompassing the datasets, implementation details, evaluation protocols, and baseline methods employed to comprehensively assess Zubov-Net.

\noindent \textbf{Datasets.}\quad We evaluate our method on four datasets: SVHN \cite{netzer2011reading}, CIFAR-10, CIFAR-100 \cite{krizhevsky2009learning}, and Tiny-ImageNet. These datasets are chosen to benchmark performance across varying complexities and to validate the scalability of our approach. The evaluation ranges from street view house numbers (SVHN, 10 classes) to fine-grained natural image classification (CIFAR-100, 100 classes). We further include Tiny-ImageNet (200 classes) to assess the method under a more challenging, large-scale classification setting.

\noindent \textbf{Evaluation Details.}\quad Model performance is evaluated by classification accuracy on SVHN, CIFAR-10, CIFAR-100, and Tiny-ImageNet. Following \cite{CUI2023576, YAN2020On}, we assess robustness through stochastic noise perturbations (Gaussian, glass, shot, impulse, speckle, motion, brightness, and contrast noises   \cite{hendrycks2018benchmarking}) and adversarial attacks. For adversarial evaluation, we employ FGSM \cite{goodfellow2014explaining}, BIM \cite{kurakin2018adversarial}, PGD \cite{madry2018towards}, APGD \cite{pmlr-v119-croce20b}, and Jitter \cite{schwinn2023exploring} with perturbation budgets $\varepsilon \in \{8/255, 16/255\}$. Iterative attacks (BIM, PGD, APGD, Jitter) use 10 iterations with step sizes scaled as $\alpha = \varepsilon/8$. All adversarial evaluations are implemented using torchattacks \cite{kim2020torchattacks}.For Zubov-Net, the adversarial attack is performed in a white-box setting where the adversary has full access to the entire pipeline, including the feature extractor, the Neural ODE, and the Lyapunov classifier. The attack objective is chosen to be the training loss $\mathcal{L}_{cla}$ and logits are defined as $\ln(\mathbf{W}(h_x(T))^{\circ -1} - \alpha \mathbf{1})$, which is equivalent to the standard cross-entropy loss applied to the final predicted probabilities $\hat{\mathbf{y}}_x$.

\noindent\textbf{Baselines.} We evaluate Zubov-Net against the following representative models:
\begin{itemize}
	\item \textbf{Standard Architectures:}
		\textbf{ResNet-18} \cite{He_2016_CVPR} serves as the standard discrete-depth baseline.  \textbf{Neural ODE} \cite{NEURIPS2018_69386f6b} provides the foundational continuous-depth formulation without explicit stability constraints.
	
	\item \textbf{Stability-enhanced Neural ODEs:}
	\textbf{LyaNet} \cite{pmlr-v162-rodriguez22a} implements Lyapunov-stable ODEs with exponential stability.
	\textbf{FxTS-Net} \cite{LUO2025107219} incorporates fixed-time stability constraints to ensure convergence within a user-defined time.
	\textbf{SODEF} \cite{NEURIPS2021_7d5430cf} enforces each clean training feature as an asymptotically stable point.
	\textbf{Proj-NODE} enforces global stability by projecting dynamics onto a stable field defined by a Lyapunov network, following \cite{NEURIPS2019_0a4bbced, Chu_Wei_Liu_Zhao_Miyatake_2024}.
\end{itemize}

All compared methods use a ResNet-18 feature extractor. For a fair comparison of the dynamics, $f(\cdot; \theta_f)$ across all methods is implemented as a 2-layer MLP with a hidden size of 256. Following their original designs: LyaNet and FxTS-Net incorporate an additional input feature modulation module; SODEF uses $\sin$ as activation; Proj-NODE, Neural ODE, and Zubov-Net employ $\mathrm{Tanh}$ activation.

\noindent  \textbf{Training Details.}\quad  Our implementation is built upon a public codebase for Lyapunov learning\footnote{\url{https://github.com/ivandariojr/LyapunovLearning}}. We use the Dopri5 solver with an absolute and relative tolerance of $1\times10^{-1}$ for training, $1\times10^{-3}$ for validation, and integration time $T=1$ with $\Gamma=10$. All models are trained using the Adam optimizer $(\beta_1=0.9, \beta_2=0.999)$ with an initial learning rate of $1\times10^{-3}$ and no weight decay. The learning rate is scheduled via MultiStepLR, decaying by a factor of $0.1$ at epoch $90$, for a total of $140$ epochs. A uniform batch size of $128$ is used for training.  Experiments were conducted on an NVIDIA RTX 4090 GPU using PyTorch 2.2.0 and CUDA 12.6.

In Zubov-Net, the Lyapunov classifier is instantiated by a PIACNN $g(\cdot)$ with layer sizes $X_{layer} = [256, 256, 1]$ and $U_{layer} = [256, 256, 0]$, and a convexity coefficient $\delta = 0.5$. Following \cite{liu2025physics}, we set $\Phi(h) = \|h\|_{\mathcal{A}}^2$ in the Zubov equation.  The coefficients in the tripartite losses (Eq. \eqref{eq:oq1} ) are $\lambda_1 = 1.5$, $\lambda_2 =  0.12$, and $\lambda_3= 0.9$. The scaling factor $\beta$ in the separation loss (Eq. \eqref{eq:3.5}) is set to $ 0.85$. The boundary sampling algorithm operates on a subset of $30$ classes per epoch. The sampling directions are strategically constructed: for each class anchor $\mathbf{c}_i$, primary directions are formed by the normalized vectors to other class anchors ($\mathbf{c}_j - \mathbf{c}_i$). Then the primary directions are augmented with 20 random directions per $(i,j)$ pair, which are normalized, scaled by a random factor in $[0, 0.5]$, and added to the primary directions before final normalization. In Algorithm \ref{alg:2}, the gradient projection  is applied for $N_2 = 5$ and $\eta_{2}=1.2$.

\begin{table*}[ht]
	\centering
	\caption{Robustness against five white-box attacks under two perturbation radii. The best results are shown in bold. Unit: \%.} 
	\label{table:attac}
	\begin{tabular}{ccccccccccccc}
		\toprule
		Dataset                  & Method               & \multicolumn{2}{c}{FGSM}        & \multicolumn{2}{c}{PGD}         & \multicolumn{2}{c}{BIM}         & \multicolumn{2}{c}{APGD}        & \multicolumn{2}{c}{Jitter}      & Avg \\
		\cmidrule(r){3-4} \cmidrule(r){5-6} \cmidrule(r){7-8} \cmidrule(r){9-10} \cmidrule(r){11-12}
		\multicolumn{1}{l}{}     & \multicolumn{1}{l}{} & 8/255          & 16/255         & 8/255          & 16/255         & 8/255          & 16/255         & 8/255          & 16/255         & 8/255          & 16/255         & \multicolumn{1}{l}{}     \\ \midrule
		
		\multirow{6}{*}{SVHN} & ResNet18             & 62.96           & 57.23           & 69.87           & 69.56          & 82.76           & 67.26          & 69.01           & 68.56           & 58.41            & 50.48            & 65.61  \\
		& Neural ODE           & 63.86           & 57.93           & 70.48           & 70.05          & 82.84           & 67.55          & 69.38           & 69.08           & 59.17            & 51.02            & 66.14                \\
		& Proj-NODE			   & 65.04 			 & 60.21 	       & 72.17  	     & 71.87  	      & 82.88 	        & 68.09 	     & 71.47 	       & 71.17 	         & 59.57 	        & 51.59 	       & 67.40 \\
		& SODEF                & 64.09           & 59.25           & 71.80           & 71.47          & 83.33           & 69.74          & 70.58           & 70.17           & 60.86            & 53.11            & 67.44                \\
		& FxTS-Net             & 66.10           & 61.35           & 72.23           & 72.07          & 90.22           & 82.48          & 72.15           & 71.64           & 63.53            & 58.09            & 71.01                \\
		& LyaNet               & 65.26           & 60.84           & 72.76           & 72.23          & 84.45           & 70.39          & 71.65           & 71.14           & 61.19            & 53.48            & 68.34                \\
		& Zubov-Net            & \textbf{71.06}  & \textbf{69.49}  & \textbf{72.83}  & \textbf{72.58} & \textbf{92.06}  & \textbf{89.36} & \textbf{72.26}  & \textbf{71.75}  & \textbf{67.38}   & \textbf{63.95}   & \textbf{74.27}       \\ \midrule
		\multirow{6}{*}{CIFAR-10} & ResNet18             & 39.66          & 30.99          & 52.48          & 51.97          & 61.35          & 40.26          & 50.43          & 50.23          & 18.30          & 6.90           & 40.26    \\
		& Neural ODE           & 40.95          & 32.37          & 52.79          & 52.23          & 62.55          & 41.38          & 50.01          & 49.97          & 18.32          & 7.02           & 40.76           \\
		& Proj-NODE  		   & 45.05 			& 38.12 		 & 57.17 		  & 56.77 		   & 60.74 			& 41.02 		 & 55.27 		  & 54.14 		   & 26.24 			& 16.81 		 & 45.13          \\
		& SODEF                & 47.54          & 40.88          & 57.36          & 57.15          & 64.01          & 43.87          & 55.42          & 55.18          & 29.95          & 20.02          & 47.14            \\
		& FxTS-Net             & 49.79          & 48.18          & 56.36          & 55.86          & 71.65          & 65.41          & 54.20          & 54.06          & 31.81          & 22.14          & 50.95                \\
		& LyaNet               & 49.30          & 48.08          & 55.23          & 55.16          & 71.49          & 66.37          & 53.08          & 52.37          & 30.28          & 20.55          & 50.19                \\
		& Zubov-Net            & \textbf{56.83} & \textbf{53.77} & \textbf{59.94} & \textbf{59.37} & \textbf{75.01} & \textbf{67.74} & \textbf{58.03} & \textbf{57.81} & \textbf{44.87} & \textbf{37.78} & \textbf{57.12}  \\ \midrule
		\multirow{6}{*}{CIFAR-100} & ResNet18             & 16.97           & 13.93           & 21.47            & 21.40           & 37.15          & 22.78           & 19.52           & 19.27           & 7.36             & 3.89             & 18.37                \\
		& Neural ODE           & 16.24           & 13.58           & 20.97           & 20.62          & 36.90           & 22.86           & 18.78           & 18.36           & 7.01             & 3.71             & 17.90                \\
		& Proj-NODE	           & 18.86 	         & 17.51 	       & 20.76 	         & 20.46 	      & 41.71 	       & 29.95 	         & 18.92 	       & 18.71 	         & 8.82 	        & 4.51 	           & 20.02 \\
		& SODEF                & 17.57           & 15.83           & 21.33           & 21.17          & 40.38          & 26.97           & 19.47           & 18.86           & 7.78             & 4.06             & 19.34                \\
		& FxTS-Net             & \textbf{19.99}  & \textbf{18.63}  & 20.69          & 20.25          & \textbf{48.90}  & \textbf{35.94}  & 18.66           & 18.26           & 12.53            & 9.67             & 22.35                \\
		& LyaNet               & 18.30            & 16.32           & 21.27           & 21.03          & 46.09          & 33.18           & 19.26           & 18.94           & 12.85            & 9.52             & 21.68                \\
		& Zubov-Net            & 19.53           & 17.29           & \textbf{22.33}  & \textbf{22.14} & 44.47          & 31.83           & \textbf{20.52}  & \textbf{20.22}  & \textbf{16.44}   & \textbf{13.02}   & \textbf{22.78}   \\ \midrule
		\multirow{6}{*}{Tiny-ImageNet}  & ResNet18   & 8.96           & 6.83           & 12.41          & 12.01          & 26.05          & 15.29          & 9.55           & 9.15           & 2.54          & 1.01          & 10.38          \\
		& Neural ODE & 9.36           & 7.22           & 11.87          & 11.36          & 26.08          & 15.30          & 9.96           & 9.21           & 2.67          & 1.08          & 10.41          \\
		& Proj-NODE  & 9.93           & 8.21           & 13.01          & 12.51          & 26.15          & 15.54          & 11.36          & 10.84          & 2.86          & 1.12          & 11.15          \\
		& SODEF      & 10.08          & 8.38           & 13.21          & 12.64          & 26.77          & 15.66          & 10.92          & 10.15          & 2.50          & 1.04          & 11.13          \\
		& FxTS-Net   & 11.79          & 11.08          & 13.13          & 12.83          & \textbf{33.08} & \textbf{26.98} & 11.10          & 10.57          & 4.76          & 3.28          & 13.86          \\
		& LyaNet     & 11.23          & 10.94          & 12.94          & 12.52          & 32.96          & 25.88          & 10.26          & 9.75           & 4.45          & 2.85          & 13.38          \\
		& Zubov-Net  & \textbf{12.13} & \textbf{11.24} & \textbf{13.52} & \textbf{13.04} & 30.36          & 23.55          & \textbf{11.97} & \textbf{11.38} & \textbf{7.22} & \textbf{4.68} & \textbf{13.91}  \\
		\bottomrule
	\end{tabular}
\end{table*}

\subsection{Accuracy and Robustness against Input Perturbations}
\label{sec:exp_main}

This subsection presents a comprehensive evaluation of Zubov-Net's core capability: maintaining high accuracy on clean data while exhibiting strong robustness against both random corruptions and adversarial attacks. We report results on four datasets of increasing complexity (SVHN, CIFAR-10, CIFAR-100, Tiny-ImageNet) in Tables \ref{table:acc_noise} and \ref{table:attac}. The key metric is the average performance (``Avg'') across all perturbation types within each category, reflecting the overall balance between accuracy and robustness.

Zubov-Net achieves the highest average accuracy across all eight stochastic noises on every dataset (Table \ref{table:acc_noise}). It also attains the highest average robustness against five white-box adversarial attacks under two perturbation radii on all datasets (Table \ref{table:attac}). The performance in both evaluations shows consistent improvement, aligning with the goal of mitigating the accuracy-robustness trade-off.

Thorough analysis of the results reveals that the performance advantage of Zubov-Net is consistent but varies in magnitude across different datasets. The lead is most pronounced on CIFAR-10, where it outperforms the strongest prior stability-based baseline by 1.11\% and 6.17\%  in average noise robustness and adversarial robustness, respectively. On the more complex CIFAR-100 and Tiny-ImageNet datasets, while Zubov-Net still secures the top average ranking, the margin narrows. This suggests that ensuring well-separated attraction basins becomes increasingly challenging as the number of classes grows, presenting a natural scalability limit for geometric control methods.

These empirical results directly reflect the core design principles of our framework. The high clean accuracy is consistent with the expectation that the classification loss, operating within the unified Lyapunov classifier architecture, successfully guides trajectories of clean inputs to converge to the correct equilibrium. The superior average robustness indicates that the Zubov-driven consistency loss effectively aligns the prescribed attraction basins (PRoAs) with the true dynamical RoAs, thereby reducing misaligned and vulnerable regions in the state space.

In summary, the results affirm that Zubov-Net's paradigm of learning aligned stability through a unified architecture, Zubov-driven matching, and active geometric control yields a model with superior and balanced performance against various perturbations. The following sections will provide further geometric evidence.

\begin{figure*}[!t] 
	\centering
	\includegraphics[width=\textwidth]{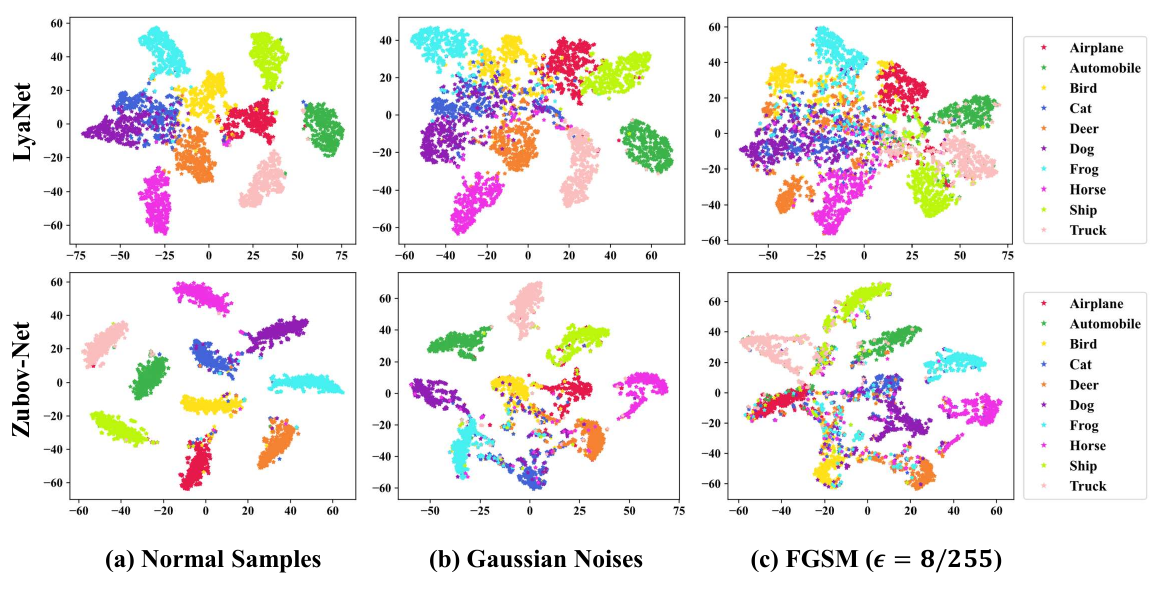} 
	\caption{TSNE visualization results on classification features of the LyaNet and Zubov-Net using CIFAR-10 with 400 random images per category. From left to right, the normal data, the FGSM attack perturbation, and the Gaussian noise. Note: Points represent individual image features; color indicates the ground-truth class label.} 
	\label{fig:tsne}
	\vspace{-10pt}
\end{figure*}

\subsection{Inspecting Decision Boundaries}
To further validate the effect of our method, we visualize the feature representations of LyaNet and Zubov-Net on CIFAR-10 under different conditions using t-SNE, including clean samples, Gaussian noise, and FGSM adversarial attacks.

As shown in Figure \ref{fig:tsne}, Zubov-Net exhibits clearer inter-class separation and more compact intra-class clustering across all settings. Under clean data, the boundaries between semantically similar classes (e.g., cats and dogs) are more distinct compared to LyaNet. When perturbed by Gaussian noise or FGSM attacks, Zubov-Net maintains better structural integrity of feature clusters, while the features of LyaNet become significantly scattered and overlapping.

These visualizations provide powerful geometric evidence for the success of our framework.  The invariant structural integrity under diverse perturbations shows that  Zubov-Net appropriately imposes stability conditions, transforming them from a theoretical constraint into an observable geometric property. This robustness stems from the synergistic action of our tripartite objective: the unified architecture and the Zubov-driven matching mechanism work to align stable attraction basins with decision regions, while the active separation control enhances the margins between them. Consequently, perturbed inputs are more likely to remain within the correct, well-separated basin, manifesting as resilient feature clusters in the visualization. This provides geometric evidence that learning aligned stability can reconcile discriminative power with robustness.

\begin{table}[!t]
	\centering
	\caption{Robustness on CIFAR-10 under strong adversarial attacks. Clean accuracy (“Orig”) and robustness against VNI, Square, and AutoAttack (AA) attacks ($\varepsilon=16/255$) are compared across stability-based methods, adversarial training methods, and the combined Zubov‑Net (TRADES).} 
	\label{table:extend_attack}
	\begin{tabular}{ccccc}
		\toprule
		\textbf{Method}   & \textbf{Orig} & \textbf{VNI} & \textbf{Square}  & \textbf{AA} \\ \midrule
		Proj-NODE         & 90.03         & 19.70        & 16.89                  & 3.31        \\
		LyaNet            & 88.34         & 22.87        & 16.15                 & 4.32        \\
		Zubov-Net         & \textbf{91.29}         & 23.11        & 24.91                  & 6.84        \\
		PGD-AT            & 87.89         & 57.11        & 37.21                 & 12.25       \\
		MART              & 87.42         & 55.17        & 42.30                & 24.64       \\
		TRADES            & 87.90         & 56.04        & 45.33                & 25.13       \\
		Zubov-Net (TRADES) & 88.81         & \textbf{58.43}        & \textbf{47.24}             &\textbf{ 25.81 }     \\ \bottomrule
	\end{tabular}
	\vspace{-10pt}
\end{table}

\subsection{Broader Robustness Evaluation and Synergy with Adversarial Training}

To assess robustness under more challenging threat models, three additional attacks are introduced: the transfer-based black-box attack VNI \cite{9578612}, where adversarial examples are generated by attacking pre-trained Neural ODE and ResNet-18 models on CIFAR-10; the query-based black-box attack Square \cite{10.1007/978-3-030-58592-1_29}; and the ensemble attack AutoAttack (AA) \cite{pmlr-v119-croce20b}. We also examine the synergy between Zubov-Net and adversarial training by combining it with TRADES, where the TRADES loss is applied to the probabilistic output $\hat{\mathbf{y}}_x = \tilde{\psi}(h_x(T); C_\psi)$ of the Lyapunov classifier. This combined model, denoted Zubov‑Net (TRADES), is compared with standard adversarial training methods (PGD‑AT, MART \cite{Wang2020Improving}, and TRADES \cite{pmlr-v97-zhang19p}).

The results in Table \ref{table:extend_attack} show the following observations. Among the three stability-based methods (Proj-NODE, LyaNet, and Zubov-Net), Zubov-Net outperforms all baselines on all three attacks: VNI (23.11\% vs. 19.70\% and 22.87\%), Square (24.91\% vs. 16.89\% and 16.15\%), and AA (6.84\% vs. 3.31\% and 4.32\%). Zubov-Net also maintains the highest clean accuracy among these methods (91.29\%).

When Zubov‑Net is fused with TRADES, the model achieves the best robustness across almost all attack types while maintaining high accuracy. The robustness derived from geometric principles differs in nature from the robustness obtained through training with adversarial data. The improvement after combining Zubov‑Net with TRADES demonstrates that its “learning aligned stability” framework is compatible and synergistic with the adversarial training.

\subsection{Computational Efficiency Analysis}
To address practical deployment concerns, we compare the computational efficiency of Zubov-Net against key baselines. Table \ref{table:efff} reports the number of parameters, training time per epoch, training memory footprint, and inference time per epoch on CIFAR-10.

Zubov-Net introduces a moderate increase in training time per epoch (48.74s) compared to the standard Neural ODE (35.39s). This overhead is attributed to the online computations essential for its active control paradigm, namely parallel boundary sampling and counterexample generation. Notably, its training cost remains lower than that of the projection-based method (Proj-NODE 57.94s), as the projection-based method requires per-step projection computations that are tightly coupled with the ODE solver's discretization.

In terms of model complexity and inference efficiency, Zubov-Net is competitive with lightweight baselines. Its parameter count (12.0M) and memory usage during training (0.2458 GB) are comparable to Neural ODE and Proj-NODE, and significantly lower than LyaNet (19.6M, 0.3867 GB). Crucially, the inference speed of Zubov-Net (1.219 s) is nearly identical to the standard Neural ODE and is substantially faster than LyaNet (5.438 s) and Proj-NODE (4.16 s). This efficiency stems from the design of its unified architecture: at inference, only a forward pass through the Neural ODE and the Lyapunov classifier is required, without additional stabilization operations.
\begin{table*}[ht]
	\centering
	\caption{Computational profile comparison on CIFAR-10. Training time, inference time, and training memory are measured per epoch (batch size = 128) and averaged over 5 epochs. Training memory is the average of memory usage per batch within the epoch.} 
	\label{table:efff}
	\begin{tabular}{ccccc}
		\toprule
		\textbf{Model} & \textbf{Params (M)} & \textbf{Train Time/Epoch (s)} & \textbf{Train Memory (GB)} & \textbf{Inference Time/Epoch (s)} \\  \midrule
		ResNet18       & 11.2                & 23.09                         & 0.2331                     & 0.688                              \\
		Neural ODE     & 11.4                & 35.39                         & 0.2405                     & 1.196                              \\
		LyaNet         & 19.6                & 41.61                         & 0.3867                     & 5.438                              \\
		Proj-NODE      & 11.6                & 57.94                         & 0.2418                     & 4.160                               \\
		Zubov-Net      & 12.0                  & 48.74                         & 0.2458                     & 1.219                             \\ \bottomrule
	\end{tabular}
\end{table*}

\begin{table*}[ht]
	\centering
	\caption{Classification accuracy and robustness for ablation study on CIFAR-10 ($\varepsilon=8/255$). ``Original'' denotes accuracy on clean data. \textsuperscript{†} The mean $\frac{dW}{dt}f(h;\theta_f)$ along trajectories is reported (more negative indicates stronger theoretical attraction). The ``Avg'' column reports the average robustness accuracy over the ten corruption and attack types (excluding clean accuracy). The best result is in bold. Unit: \%. } 
	\label{table:ab1}
	\begin{tabular}{lccccccccccccc}
		\toprule
		\textbf{Methods} & \textbf{Orig} & \textbf{Gauss} & \textbf{Glass} & \textbf{Shot} & \textbf{Speckle} & \textbf{Motion} & \textbf{FGSM} & \textbf{PGD} & \textbf{BIM} & \textbf{APGD} & \textbf{Jitter} & \textbf{Avg} & \textbf{$\frac{dW}{dt}f(h;\theta_f)$\textsuperscript{†}} \\ \midrule
		
		$\mathcal{L}_{cla}$ + $\mathcal{L}_{FC}$  & 90.87 & 67.99 & 72.42 & 63.23 & 78.69 & 82.28 & 52.53 & 57.58 & 71.74 & 55.86 & 42.94 & 64.53 & -- \\
		w/o $\mathcal{L}_{con}$ & 91.05 & 69.78 & 74.51 & 65.28 & 77.78 & 82.68 & 54.61 & 58.31 & 73.04 & 56.54 & 44.37 & 65.69 & -- \\ 
		w/o $\mathcal{L}_{sep}$ & 90.74 & 69.52 & 75.11 & 64.96 & 77.45 & 82.61 & 54.78 & 58.36 & 74.05 & 56.45 & 43.87 & 65.72 & -0.4 \\
		w/o $\mathcal{L}_{sep}$ (Ineq) & 89.81 & 69.46 & 70.52 & 64.03 & 76.47 & 80.30 & 53.73 & 57.63 & 71.59 & 56.31 & 43.25 & 64.33   & \textbf{-0.8} \\ \midrule
		Ours & \textbf{91.29} & \textbf{71.06} & \textbf{76.43} & \textbf{66.71} & \textbf{78.80} & \textbf{82.86} & \textbf{56.83} & \textbf{59.94} & \textbf{75.01} & \textbf{58.03} & \textbf{44.87} & \textbf{67.05} &  --\\  \midrule
		ICNN  & 90.58 & 68.15 & 74.39 & 64.73 & 77.21 & 81.10 & 54.06 & 57.71 & 73.79 & 56.74 & 42.86 & 65.07 & --\\
		PICNN  & 91.14 & 69.79 & 75.31 & 65.06 & 77.42 & 81.27 & 55.46 & 58.14 & 74.17 & 57.25 & 43.16 & 65.70 & --\\ 
		\bottomrule   
	\end{tabular}
\end{table*}

\subsection{Ablation Experiments}

To further evaluate the effectiveness and practicality of our framework, we conduct ablation experiments on CIFAR-10, with more experiments provided in the Appendix \ref{app:exp}.

\noindent \textbf{Loss component ablation.} As shown in Table \ref{table:ab1}, using the classification losses ($\mathcal{L}_{cla}$ and $\mathcal{L}_{FC}$) achieves high clean accuracy but performs poorly under all perturbations. Removing the separation loss ($\mathcal{L}_{sep}$) leads to a clear drop in adversarial robustness, especially against iterative attacks like PGD and BIM. Removing the consistency loss ($\mathcal{L}_{con}$) mainly reduces performance under stochastic noise perturbations such as Glass and Shot. The full model achieves the best overall performance, indicating that both $\mathcal{L}_{sep}$ and $\mathcal{L}_{con}$ are necessary for improving robustness against adversarial and stochastic perturbations, respectively.

\noindent \textbf{Further analysis on the Zubov-driven alignment.} We also test a variant (shown in Table \ref{table:ab1}) where the Zubov equation in $\mathcal{L}_{con}$ is relaxed to an inequality using ReLU. This modification allows the Lyapunov derivative to be more negative, as indicated by the mean $\frac{dW}{dt}f(h;\theta_f)$ value of ‑0.8 (compared to ‑0.4 for the standard w/o $\mathcal{L}_{sep}$). However, as shown in the row ``w/o $\mathcal{L}_{sep}$ (Ineq)'', this variant exhibits lower clean accuracy (89.81\% vs. 90.74\%) and lower robustness under all evaluated perturbation types compared to the standard w/o $\mathcal{L}_{sep}$ configuration. Within this experimental setting, the stronger stability condition does not lead to improved robustness; instead, it is associated with degraded performance across both accuracy and robustness metrics.

\noindent \textbf{Lyapunov classifier architecture analysis.}
We further investigate the impact of the Lyapunov classifier architecture by comparing three designs: ICNN (using $L$ independent networks), PICNN (a single partially input-convex network), and our proposed PIACNN (which incorporates an input-attention mechanism into PICNN). As shown in Table \ref{table:ab1}, ICNN achieves an average robustness accuracy of 65.07\%, while PICNN reaches 65.70\%. Our PIACNN architecture attains the highest average robustness (67.05\%), outperforming PICNN by 1.35\%. This suggests that the attention mechanism in PIACNN enhances the discrimination of the Lyapunov function by focusing on equilibrium-relevant features, which in turn refines the geometry of the prescribed regions of attraction (PRoAs), thereby mitigating the accuracy-robustness trade-offs.

\section{Conclusion}
This work addresses the accuracy-robustness trade-off in Neural ODEs by resolving the fundamental misalignment between stability regions and decision boundaries. The proposed Zubov-Net framework introduces a unified architecture where learnable Lyapunov functions serve as multi-class classifiers, ensuring prescribed regions of attraction (PRoAs) inherently align with classification objectives. Within this architecture, a Zubov-driven matching mechanism reformulates Zubov’s equation into a differentiable consistency loss, aligning PRoAs with the true dynamical regions of attraction (RoAs). This alignment enables an active control paradigm: the geometry of RoAs is directly shaped by optimizing PRoAs through a tripartite loss and a parallel boundary sampling algorithm. The Lyapunov classifier is implemented via a Partially Input-Attention-based Convex Neural Network (PIACNN), whose attention mechanism focuses on equilibrium-relevant features while acting as weight normalization.

Theoretically, minimizing the tripartite loss guarantees PRoA-RoA alignment, trajectory containment, and non-overlapping PRoAs. A certified robustness radius is derived, linking the aligned basin geometry directly to robustness through Lipschitz continuity. Furthermore, by adopting a pointwise separation strategy, we achieve tighter probability bounds and lower dimensionality requirements, justifying the convex design of the Lyapunov classifier. Experimentally, Zubov-Net maintains high accuracy while significantly improving robustness against stochastic noises and adversarial attacks on SVHN, CIFAR-10, CIFAR-100, and Tiny-ImageNet. Its effectiveness is validated under strong adaptive attacks (AutoAttack), black-box settings, and synergistic combination with adversarial training (TRADES). Ablation studies validate each component, and efficiency analysis shows competitive training/inference costs.

Future Work will proceed along three axes. First, we will apply the aligned stability framework to time-series forecasting, reinforcement learning, and other dynamical systems beyond image classification. Second, we will extend Zubov-Net to Graph Neural Networks, Transformers, and NAS-based robust designs for domain-specific tasks. Third, we will deepen the theory to cover scenarios where the consistency loss is only approximately minimized, and generalize robustness guarantees to nonlinear systems with complex attractors, including limit cycles and non-convex stability regions.

\bibliographystyle{IEEEtran}
\bibliography{ref}

\begin{thebibliography}{10}
\providecommand{\url}[1]{#1}
\csname url@samestyle\endcsname
\providecommand{\newblock}{\relax}
\providecommand{\bibinfo}[2]{#2}
\providecommand{\BIBentrySTDinterwordspacing}{\spaceskip=0pt\relax}
\providecommand{\BIBentryALTinterwordstretchfactor}{4}
\providecommand{\BIBentryALTinterwordspacing}{\spaceskip=\fontdimen2\font plus
\BIBentryALTinterwordstretchfactor\fontdimen3\font minus
  \fontdimen4\font\relax}
\providecommand{\BIBforeignlanguage}[2]{{%
\expandafter\ifx\csname l@#1\endcsname\relax
\typeout{** WARNING: IEEEtran.bst: No hyphenation pattern has been}%
\typeout{** loaded for the language `#1'. Using the pattern for}%
\typeout{** the default language instead.}%
\else
\language=\csname l@#1\endcsname
\fi
#2}}
\providecommand{\BIBdecl}{\relax}
\BIBdecl

\bibitem{11018619}
Q.~Zhao, G.~Li, B.~He, and R.~Shen, ``Deep learning for low-light vision: A
  comprehensive survey,'' \emph{IEEE Transactions on Neural Networks and
  Learning Systems}, vol.~36, no.~9, pp. 15\,685--15\,705, 2025.

\bibitem{10905032}
P.~Cheng, Z.~Wu, W.~Du, H.~Zhao, W.~Lu, and G.~Liu, ``Backdoor attacks and
  countermeasures in natural language processing models: A comprehensive
  security review,'' \emph{IEEE Transactions on Neural Networks and Learning
  Systems}, vol.~36, no.~8, pp. 13\,628--13\,648, 2025.

\bibitem{10363393}
M.~Han, K.~Wong, J.~Euler-Rolle, L.~Zhang, and R.~K. Katzschmann, ``Robust
  learning-based control for uncertain nonlinear systems with validation on a
  soft robot,'' \emph{IEEE Transactions on Neural Networks and Learning
  Systems}, vol.~36, no.~1, pp. 510--524, 2025.

\bibitem{10530438}
L.~Zhang, N.~Yang, Y.~Sun, and P.~S. Yu, ``Provable unrestricted adversarial
  training without compromise with generalizability,'' \emph{IEEE Transactions
  on Pattern Analysis and Machine Intelligence}, vol.~46, pp. 8302--8319, 2024.

\bibitem{8611298}
X.~Yuan, P.~He, Q.~Zhu, and X.~Li, ``Adversarial examples: Attacks and defenses
  for deep learning,'' \emph{IEEE Transactions on Neural Networks and Learning
  Systems}, vol.~30, no.~9, pp. 2805--2824, 2019.

\bibitem{NEURIPS2018_69386f6b}
R.~T.~Q. Chen, Y.~Rubanova, J.~Bettencourt, and D.~K. Duvenaud, ``Neural
  ordinary differential equations,'' in \emph{Advances in Neural Information
  Processing Systems}, vol.~31, 2018, p. 6572–6583.

\bibitem{9157003}
X.~Liu, T.~Xiao, S.~Si, Q.~Cao, S.~Kumar, and C.-J. Hsieh, ``How does noise
  help robustness? explanation and exploration under the neural {SDE}
  framework,'' in \emph{2020 IEEE/CVF Conference on Computer Vision and Pattern
  Recognition (CVPR)}, 2020, pp. 279--287.

\bibitem{9035109}
F.~Carrara, R.~Caldelli, F.~Falchi, and G.~Amato, ``On the robustness to
  adversarial examples of neural {ODE} image classifiers,'' in \emph{2019 IEEE
  International Workshop on Information Forensics and Security (WIFS)}, 2019,
  pp. 1--6.

\bibitem{YAN2020On}
H.~YAN, J.~DU, V.~TAN, and J.~FENG, ``On robustness of neural ordinary
  differential equations,'' in \emph{International Conference on Learning
  Representations}, 2020.

\bibitem{CUI2023576}
W.~Cui, H.~Zhang, H.~Chu, P.~Hu, and Y.~Li, ``On robustness of neural {ODEs}
  image classifiers,'' \emph{Information Sciences}, vol. 632, pp. 576--593,
  2023.

\bibitem{pmlr-v162-rodriguez22a}
I.~D.~J. Rodriguez, A.~Ames, and Y.~Yue, ``{L}ya{N}et: A {L}yapunov framework
  for training neural {ODE}s,'' in \emph{Proceedings of the 39th International
  Conference on Machine Learning}, vol. 162, 2022, pp. 18\,687--18\,703.

\bibitem{huang2022fi}
Y.~Huang, I.~D.~J. Rodriguez, H.~Zhang, Y.~Shi, and Y.~Yue, ``{FI-ODE}:
  Certifiably robust forward invariance in neural {ODEs},'' \emph{arXiv
  preprint arXiv:2210.16940}, 2022.

\bibitem{LUO2025107219}
C.~Luo, Y.~Zou, W.~Li, and N.~Huang, ``{FxTS-Net}: Fixed-time stable learning
  framework for neural {ODEs},'' \emph{Neural Networks}, vol. 185, p. 107219,
  2025.

\bibitem{NEURIPS2019_0a4bbced}
J.~Z. Kolter and G.~Manek, ``Learning stable deep dynamics models,'' in
  \emph{Advances in Neural Information Processing Systems}, vol.~32, 2019.

\bibitem{Chu_Wei_Liu_Zhao_Miyatake_2024}
H.~Chu, S.~Wei, T.~Liu, Y.~Zhao, and Y.~Miyatake, ``Lyapunov-stable deep
  equilibrium models,'' \emph{Proceedings of the AAAI Conference on Artificial
  Intelligence}, vol.~38, pp. 11\,615--11\,623, 2024.

\bibitem{10.24963/ijcai.2024/428}
H.~Chu, Y.~Miyatake, W.~Cui, S.~Wei, and D.~Furihata, ``Structure-preserving
  physics-informed neural networks with energy or lyapunov structure,'' in
  \emph{Proceedings of the Thirty-Third International Joint Conference on
  Artificial Intelligence}, 2024.

\bibitem{NEURIPS2023_0a443a00}
K.~Zhao, Q.~Kang, Y.~Song, R.~She, S.~Wang, and W.~P. Tay, ``Adversarial
  robustness in graph neural networks: A hamiltonian approach,'' in
  \emph{Advances in Neural Information Processing Systems}, vol.~36, 2023, pp.
  3338--3361.

\bibitem{NEURIPS2021_7d5430cf}
Q.~Kang, Y.~Song, Q.~Ding, and W.~P. Tay, ``Stable neural {ODE} with
  {Lyapunov-Stable} equilibrium points for defending against adversarial
  attacks,'' in \emph{Advances in Neural Information Processing Systems},
  vol.~34, 2021, pp. 14\,925--14\,937.

\bibitem{NEURIPS2022_299a08ee}
X.~Li, Z.~Xin, and W.~Liu, ``Defending against adversarial attacks via neural
  dynamic system,'' in \emph{Advances in Neural Information Processing
  Systems}, vol.~35, 2022, pp. 6372--6383.

\bibitem{liu2025physics}
J.~Liu, Y.~Meng, M.~Fitzsimmons, and R.~Zhou, ``Physics-informed neural network
  {Lyapunov} functions: {PDE} characterization, learning, and verification,''
  \emph{Automatica}, vol. 175, p. 112193, 2025.

\bibitem{kang2023data}
W.~Kang, K.~Sun, and L.~Xu, ``Data-driven computational methods for the domain
  of attraction and {Zubov's} equation,'' \emph{IEEE Transactions on Automatic
  Control}, vol.~69, pp. 1600--1611, 2023.

\bibitem{li2025twostage}
H.~Li, X.~Zhong, B.~Hu, and H.~Zhang, ``Two\nobreakdash-stage learning of
  stabilizing neural controllers via zubov sampling and iterative domain
  expansion,'' in \emph{The Thirty-ninth Annual Conference on Neural
  Information Processing Systems}, 2025.

\bibitem{NEURIPS2023_a45b205c}
S.~Pfrommer, B.~Anderson, J.~Piet, and S.~Sojoudi, ``Asymmetric certified
  robustness via feature-convex neural networks,'' in \emph{Advances in Neural
  Information Processing Systems}, vol.~36, 2023, pp. 52\,365--52\,400.

\bibitem{Latorre2020Lipschitz}
F.~Latorre, P.~Rolland, and V.~Cevher, ``Lipschitz constant estimation of
  neural networks via sparse polynomial optimization,'' in \emph{International
  Conference on Learning Representations}, 2020.

\bibitem{oh2024stable}
Y.~Oh, D.~Lim, and S.~Kim, ``Stable neural stochastic differential equations in
  analyzing irregular time series data,'' in \emph{The Twelfth International
  Conference on Learning Representations}, 2024.

\bibitem{zubov1961methods}
V.~I. Zubov, \emph{Methods of AM Lyapunov and their application}.\hskip 1em
  plus 0.5em minus 0.4em\relax US Atomic Energy Commission, 1961.

\bibitem{pmlr-v145-huang22a}
Y.~Huang, Y.~Yu, H.~Zhang, Y.~Ma, and Y.~Yao, ``Adversarial robustness of
  stabilized neural {ODE} might be from obfuscated gradients,'' in
  \emph{Proceedings of the 2nd Mathematical and Scientific Machine Learning
  Conference}, vol. 145, 2022, pp. 497--515.

\bibitem{9809979}
M.~Zakwan, L.~Xu, and G.~Ferrari-Trecate, ``Robust classification using
  contractive hamiltonian neural {ODEs},'' \emph{IEEE Control Systems Letters},
  vol.~7, pp. 145--150, 2023.

\bibitem{de2025improving}
A.~De~Marinis, N.~Guglielmi, S.~Sicilia, and F.~Tudisco, ``Improving the
  robustness of neural odes with minimal weight perturbation,'' \emph{arXiv
  preprint arXiv:2501.10740}, 2025.

\bibitem{Kang_Zhao_Song_Xie_Zhao_Wang_She_Tay_2024}
Q.~Kang, K.~Zhao, Y.~Song, Y.~Xie, Y.~Zhao, S.~Wang, R.~She, and W.~P. Tay,
  ``Coupling graph neural networks with fractional order continuous dynamics: A
  robustness study,'' \emph{Proceedings of the AAAI Conference on Artificial
  Intelligence}, vol.~38, pp. 13\,049--13\,058, 2024.

\bibitem{Cui_Kang_Li_Zhao_Tay_Deng_Li_2025}
W.~Cui, Q.~Kang, X.~Li, K.~Zhao, W.~P. Tay, W.~Deng, and Y.~Li, ``Neural
  variable-order fractional differential equation networks,'' vol.~39, pp.
  16\,109--16\,117, 2025.

\bibitem{pmlr-v70-amos17b}
B.~Amos, L.~Xu, and J.~Z. Kolter, ``Input convex neural networks,'' in
  \emph{Proceedings of the 34th International Conference on Machine Learning},
  vol.~70, 2017, pp. 146--155.

\bibitem{kidger2021hey}
P.~Kidger, R.~T.~Q. Chen, and T.~J. Lyons, ``{"Hey, that’s not an ODE"}:
  Faster {ODE} adjoints via seminorms,'' in \emph{Proceedings of the 38th
  International Conference on Machine Learning}, vol. 139, 2021, pp.
  5443--5452.

\bibitem{10045756}
T.~Matsubara, Y.~Miyatake, and T.~Yaguchi, ``The symplectic adjoint method:
  Memory-efficient backpropagation of neural-network-based differential
  equations,'' \emph{IEEE Transactions on Neural Networks and Learning
  Systems}, vol.~35, no.~8, pp. 10\,526--10\,538, 2024.

\bibitem{balunovic2020adversarial}
M.~Balunovi{\'c} and M.~Vechev, ``Adversarial training and provable defenses:
  Bridging the gap,'' in \emph{International Conference on Learning
  Representations}, 2020.

\bibitem{de2022ibp}
A.~De~Palma, R.~Bunel, K.~Dvijotham, M.~P. Kumar, and R.~Stanforth, ``{IBP}
  regularization for verified adversarial robustness via branch-and-bound,''
  \emph{arXiv preprint arXiv:2206.14772}, 2022.

\bibitem{netzer2011reading}
\BIBentryALTinterwordspacing
Y.~Netzer, T.~Wang, A.~Coates, A.~Bissacco, B.~Wu, A.~Y. Ng \emph{et~al.},
  ``Reading digits in natural images with unsupervised feature learning,'' in
  \emph{NIPS workshop on deep learning and unsupervised feature learning},
  2011. [Online]. Available: \url{http://ufldl.stanford.edu/housenumbers}
\BIBentrySTDinterwordspacing

\bibitem{krizhevsky2009learning}
\BIBentryALTinterwordspacing
A.~Krizhevsky \emph{et~al.}, ``Learning multiple layers of features from tiny
  images,'' 2009. [Online]. Available:
  \url{https://www.cs.toronto.edu/~kriz/cifar.html}
\BIBentrySTDinterwordspacing

\bibitem{hendrycks2018benchmarking}
D.~Hendrycks and T.~Dietterich, ``Benchmarking neural network robustness to
  common corruptions and perturbations,'' in \emph{International Conference on
  Learning Representations}, 2019.

\bibitem{goodfellow2014explaining}
I.~J. Goodfellow, J.~Shlens, and C.~Szegedy, ``Explaining and harnessing
  adversarial examples,'' \emph{arXiv preprint arXiv:1412.6572}, 2014.

\bibitem{kurakin2018adversarial}
A.~Kurakin, I.~J. Goodfellow, and S.~Bengio, \emph{Adversarial examples in the
  physical world}.\hskip 1em plus 0.5em minus 0.4em\relax Chapman and Hall/CRC,
  2018.

\bibitem{madry2018towards}
A.~Madry, A.~Makelov, L.~Schmidt, D.~Tsipras, and A.~Vladu, ``Towards deep
  learning models resistant to adversarial attacks,'' in \emph{International
  Conference on Learning Representations}, 2018.

\bibitem{pmlr-v119-croce20b}
F.~Croce and M.~Hein, ``Reliable evaluation of adversarial robustness with an
  ensemble of diverse parameter-free attacks,'' in \emph{Proceedings of the
  37th International Conference on Machine Learning}, vol. 119, 2020, pp.
  2206--2216.

\bibitem{schwinn2023exploring}
L.~Schwinn, R.~Raab, A.~Nguyen, D.~Zanca, and B.~Eskofier, ``Exploring
  misclassifications of robust neural networks to enhance adversarial
  attacks,'' \emph{Applied intelligence}, vol.~53, pp. 19\,843--19\,859, 2023.

\bibitem{kim2020torchattacks}
H.~Kim, ``Torchattacks: A pytorch repository for adversarial attacks,''
  \emph{arXiv preprint arXiv:2010.01950}, 2020.

\bibitem{He_2016_CVPR}
K.~He, X.~Zhang, S.~Ren, and J.~Sun, ``Deep residual learning for image
  recognition,'' in \emph{Proceedings of the IEEE Conference on Computer Vision
  and Pattern Recognition (CVPR)}, 2016, pp. 770--778.

\bibitem{9578612}
X.~Wang and K.~He, ``Enhancing the transferability of adversarial attacks
  through variance tuning,'' in \emph{2021 IEEE/CVF Conference on Computer
  Vision and Pattern Recognition (CVPR)}, 2021, pp. 1924--1933.

\bibitem{10.1007/978-3-030-58592-1_29}
M.~Andriushchenko, F.~Croce, N.~Flammarion, and M.~Hein, ``Square attack: A
  query-efficient black-box adversarial attack via random search,'' in
  \emph{Computer Vision -- ECCV 2020}, 2020, pp. 484--501.

\bibitem{Wang2020Improving}
Y.~Wang, D.~Zou, J.~Yi, J.~Bailey, X.~Ma, and Q.~Gu, ``Improving adversarial
  robustness requires revisiting misclassified examples,'' in
  \emph{International Conference on Learning Representations}, 2020.

\bibitem{pmlr-v97-zhang19p}
H.~Zhang, Y.~Yu, J.~Jiao, E.~Xing, L.~E. Ghaoui, and M.~Jordan, ``Theoretically
  principled trade-off between robustness and accuracy,'' in \emph{Proceedings
  of the 36th International Conference on Machine Learning}, vol.~97, 2019, pp.
  7472--7482.

\bibitem{boyd2004convex}
S.~P. Boyd and L.~Vandenberghe, \emph{Convex optimization}.\hskip 1em plus
  0.5em minus 0.4em\relax Cambridge university press, 2004.

\bibitem{Khalil:1173048}
H.~K. Khalil, \emph{{Nonlinear systems; 3rd ed.}}\hskip 1em plus 0.5em minus
  0.4em\relax Upper Saddle River, NJ: Prentice-Hall, 2002.

\end{thebibliography}
\begin{IEEEbiographynophoto}{Nan-jing Huang }
	 has a Ph.D. in Mathematics from Sichuan University, China.  Since 1998,  he is Professor in Mathematics at Sichuan University, China.  He is in the editorial board of Numerical Algebra, Control and Optimization. His research interests include variational inequality and complementarity problems, neural networks and nonlinear optimization.
\end{IEEEbiographynophoto}

\begin{IEEEbiographynophoto}{Chaoyang Luo}
	received the M.S. degree from the School of Mathematical Sciences, Chongqing Normal University, China. He is currently pursuing the Ph.D. degree in Mathematics at Sichuan University, China. His current research focuses on adversarial security in machine learning, robustness of deep learning models, and stability control theory for dynamical systems.
\end{IEEEbiographynophoto}
\begin{IEEEbiographynophoto}{Yan Zou}
	received the M.S. degree in College of Computer Science And Engineering, Chongqing University of Technology, China. She is currently serving as an Assistant Researcher	at Yibin University. Her research interests include AI security, 
	deepfake detection, and low-quality image recognition.
\end{IEEEbiographynophoto}

\onecolumn
{\appendices
\section{Parallel Boundary Sampling Algorithm (Algorithm \ref{alg:parallel}): Details and Complexity Analysis}
\label{app:Alg2}
\noindent \textbf{Algorithm Core.}\quad Algorithm \ref{alg:parallel} parallelly samples approximate level-set (boundary) points of a given strongly convex Lyapunov function $W$ at a target value $\rho$. For a set of $d^h$ predetermined search directions $\{\hat{q}_i\}$ centered at an equilibrium $c$, the algorithm finds, in parallel for each direction $i$, a scalar $s_i$ such that $W(c + s_i \hat{q}_i) \approx \rho$.

The algorithm is essentially a parallel adaptive binary search with early termination. It adjusts the current path length $s_i$ and an adaptive step size $a_i$ for each direction.  In every iteration, all active directions simultaneously propose a new point $P = c + (s + a) \circ \hat{Q}$. Based on comparing the function values $W(P)$ with the target interval $[\rho-\epsilon, \rho+\epsilon]$, a direction mask $M$ is generated (1 for outward move, -1 for inward move, 0 for converged). The step‑size update rule $\Lambda = \frac{|M + \operatorname{sign}(a)| + 2}{4}$ implements the classic bisection strategy: when the move direction agrees with the sign of the step (i.e., continuing exploration along the current direction), the step size remains unchanged ($\Lambda=1$); when a reversal is needed (the target interval has been overshot), the step size is halved ($\Lambda=0.5$) and its sign is flipped. This operation converts logical branching into a unified tensor calculation, enabling efficient parallel execution.

\noindent \textbf{Computational Complexity.}\quad Let $n$ be the maximum number of iterations and $N = d^h$ the initial number of directions. In each iteration, the dominant cost is:
\begin{itemize}
	\item \textbf{Forward propagation:} evaluating $W(P)$ for the current batch of active points in parallel. The complexity is $O(C_W)$, where $C_W$ denotes the cost of evaluating the batch (not a single point).
	\item \textbf{Comparison and logical operations:} generating the mask $M$, updating the step sizes, and pruning the active set $\mathcal{B}$ are all vectorized operations of cost $O(|\mathcal{B}|)$. This operation is almost completely overlapped with the forward propagation of $W$ and effectively fused into the parallel computation graph without becoming a performance bottleneck.
\end{itemize}

Because the algorithm processes all still-active directions in parallel, the wall-clock time per iteration is dominated by the batch forward-propagation cost, $C_W$. Consequently, even in the worst case where all directions require the full $n$ iterations, the total computational overhead is approximately $O(n \cdot C_W)$, with $C_W$ already accounting for parallelization over the batch size $|\mathcal{B}|$. In practice, thanks to the exponential convergence of binary search and the monotonicity of the strongly convex function values, most directions converge early, so the active set size $|\mathcal{B}|$ shrinks rapidly, further reducing the actual computational burden.

\noindent \textbf{Convergence Guarantee.}\quad For any fixed search direction, the process reduces to finding the root of $W(\cdot) = \rho$ along the ray $\{c + s \hat{q} \mid s \ge 0\}$. Strong convexity of $\mathcal{V}$ and strict monotonicity of $W$  guarantees monotonicity along this ray. The adaptive step‑size update ensures that once the function value crosses the target interval $[\rho-\epsilon, \rho+\epsilon]$, the subsequent search interval (the range of possible $s$ values) shrinks by a factor of $0.5$ exponentially. Hence, for a prescribed tolerance $\epsilon$, the algorithm converges after at most $O(\log_2(\Delta / \epsilon))$ iterations, where $\Delta$ is the length of the initial guess interval. This logarithmic iteration bound guarantees the practical efficiency of the algorithm.

\section{Additional proof details}

\subsection{Proof of Proposition \ref{propos:1} }
\label{app:pro1}

\begin{proof}
	We prove the convexity of $g_k(x, c)$ by mathematical induction over the network depth $k$. It is easy to see the following fact: for any convex non-decreasing function $\sigma$, nonnegative vector $a$, and vector-valued convex function $G(x) = [g^j(x)]_{j=1}^d$, the composition $\sigma(a^\top G(x) + b)$ is convex in $x$. This follows because $a^\top G(x) + b$ is convex when $a \geq 0$ (as nonnegative sums of convex functions preserve convexity), and because composition with a convex non-decreasing function $\sigma$ preserves convexity \cite{boyd2004convex}. Thus, for $k=1$, the first layer output is  
	\begin{equation*}
		g_1(x, c) = \sigma_0\Big( A_0^{(z)} \big( \mathrm{softmax}(A_0^{(zu)} c + b_0^{(z)}) \circ x \big) + A_0^{(u)} c + b_0 \Big).
	\end{equation*}
   This shows the convexity holds since the argument of $\sigma_0$ is affine in $x$ (because the softmax term is independent of $x$), and $\sigma_0$ is convex and non-decreasing by assumption. Now assume $g_k(x, c)$ is convex in $x$. We need to show that $g_{k+1}(x, c)$ is convex in $x$. To this end, consider the $(k+1)$-th layer:
	\begin{align*}
		g_{k+1}(x, c) = \sigma_k\Big( & A_k^{(z)} \big( \mathrm{softmax}(A_k^{(zu)} u_k + b_k^{(z)}) \circ G_k(x, c) \big) \nonumber \\
		+ & A_k^{(x)} \big( \mathrm{softmax}(A_k^{(xu)} u_k + b_k^{(x)}) \circ x \big) + A_k^{(u)} u_k + b_k \Big),
	\end{align*}
	where $G_k(x, c) = [g_k^j(x, c)]_{j=1}^{d_k}$. By the induction hypothesis, each $g_k^j(x, c)$ is convex. The term $A_k^{(z)} (\mathrm{softmax}(\cdots) \circ G_k(x, c))$ is convex because element-wise multiplication with nonnegative softmax outputs preserves convexity, and matrix multiplication by nonnegative $A_k^{(z)}$ maintains the convexity. The term $A_k^{(x)} (\mathrm{softmax}(\cdots) \circ x)$ is affine in $x$ and so is convex. The sum of these convex terms remains convex, and the outer composition with convex non-decreasing $\sigma_k$ preserves the convexity. Therefore, $g_{k+1}(x, c)$ is convex in $x$. By mathematical induction, the convexity property holds for all $k \geq 1$.
\end{proof}

\subsection{Proof of Proposition \ref{pro:convex-sep}}
\label{app:pro2}

\begin{proof}
Throughout the proof, we denote the cardinality of a set $S$ by $|S|$. For the reader's convenience, we also recall that, for $n \in \mathbb{N}$, the symmetric group $S_n$ consists of all permutations (i.e., bijections) on the set $\{1, 2, \ldots, n\}$, and that $|S_n| = n!$. Consider first the case where $d \ge N +1$. For each $y^{(j)} \in X_2$, consider the extended set $\{x^{(1)}, \dots, x^{(N)}, y^{(j)}\} \subset \mathbb{R}^d$. Define the label vector $b^{(j)} \in \mathbb{R}^{N+1}$:
\begin{equation*}
	b^{(j)}_i = \begin{cases} 
		1 & \text{for } i = 1,\dots,N; \\
		-1 & \text{for } i = N+1.
	\end{cases}
\end{equation*}

Form the data matrix $A_j = \left[ x^{(1)\top}, \ldots, x^{(N)\top}, y^{(j)\top} \right]^\top \in \mathbb{R}^{(N+1) \times d}$, which has full row rank $\operatorname{rank}(A_j) = N+1$ almost surely since $d \ge N+1$ and the points are i.i.d. and uniformly distributed on $[-1,1]^d$. Therefore, the linear system $A_j a_j = b^{(j)}$ has a solution $a_j \in \mathbb{R}^d$ with probability 1. From the solution $a_j$, define the open half-space as $H_j^+ = \{ x \in \mathbb{R}^d \mid a_j^\top x > 0 \}$. By construction, we have $a_j^\top y^{(j)} = -1 < 0$ and $	a_j^\top x^{(i)} = 1 > 0 $ for all $ x^{(i)} \in X_1 $.Hence $X_1 \subset H_j^+$ and $y^{(j)} \notin H_j^+$ for each $j=1,\dots,M$. We define $C$ as the intersection of these half-spaces: $C = \bigcap_{j=1}^M H_j^+$. 

Now we verify the required properties. First, $C$ is convex because each $H_j^+$ is convex (as a half-space) and the intersection of convex sets is convex. Second, $C$ contains $X_1$ because for any $x^{(i)} \in X_1$ and for every $j$, we have $a_j^\top x^{(i)} > 0$ and hence $x^{(i)} \in H_j^+$, so $X_1 \subset C$. Third, $C$ excludes $X_2$ because for each $y^{(j)} \in X_2$, we have $a_j^\top y^{(j)} < 0$ and thus $y^{(j)} \notin H_j^+$. Then, $y^{(j)} \notin C$, which implies $X_2 \cap C = \emptyset$. Thus, $X_2 \cap \operatorname{co}(X_1) = \emptyset$ holds with probability 1 because for each $j$, the linear system $A_j a_j = b^{(j)}$ has a solution almost surely, and there are finitely many $j$ (namely $M$). 

Next, we consider the general case where $d \in \mathbb{N}$ and in general it may be the case that $d < N+1$. Our goal is to bound the probability $P = \mathbb{P}\left(X_2 \cap \operatorname{co}(X_1) = \emptyset \right)$. Following the key insight from the $d \ge N+1$ case, we decompose the probability using conditional independence and Jensen's inequality (since $M\ge1$ and $t\mapsto t^M$ is convex on $[0,1]$) 
\begin{equation}
	P = \mathbb{E}_{X_1} \left[ \prod_{j=1}^M \mathbb{P}\!\left(y^{(j)} \notin \operatorname{co}(X_1) \mid X_1\right)\right] \ge \prod\limits_{j = 1}^M {\mathbb{P}\left( y^{(j)} \notin  \operatorname{co}(X_1) \right)}. \label{eq:q}
\end{equation}

For any $y^{(j)}  \in X_2$, we adopt the per-coordinate separation to bound the probability $\mathbb{P}\left( y^{(j)} \notin  \operatorname{co}(X_1) \right)$. Defining coordinate separation events for each dimension $k = 1, \dots, d$, one has $A_k^{(j)} = \left\{ y_k^{(j)} < \min_{i=1}^N x_k^{(i)} \right\}$ and $B_k^{(j)} = \left\{ y_k^{(j)} > \max_{i=1}^N x_k^{(i)} \right\}$. If any event $A_k^{(j)}$ or $B_k^{(j)}$ occurs, then $y^{(j)}$ lies outside $\operatorname{co}(X_1)$ since it is separated from the convex hull using coordinate $k$. Thus, the probability lower bound is  
\begin{align}
	\mathbb{P}\left(y^{(j)} \notin \operatorname{co}(X_1)\right) & \geq \mathbb{P}\left( \bigcup_{k=1}^d \left( A_k^{(j)} \cup B_k^{(j)} \right) \right) \nonumber \\
	&= 1 - \mathbb{P}\left( \bigcap_{k=1}^d \left( (A_k^{(j)})^c \cap (B_k^{(j)})^c \right) \right) \nonumber\\
   & =1 - \mathbb{P}\left( \min_i x_k^{(i)} \leq y_k^{(j)} \leq \max_i x_k^{(i)} \quad \forall k \right)  \nonumber \\
	& = 1 - \prod_{k=1}^d \mathbb{P}\left( \min_i x_k^{(i)} \leq y_k^{(j)} \leq \max_i x_k^{(i)} \right) \quad \text{(By coordinate independence)} \nonumber\\
	& = 1 - \prod_{k=1}^d \left(1 -  \mathbb{P}\left( A_k^{(j)} \cup B_k^{(j)} \right)\right). \label{eq:q1}
\end{align} 

Since $A_k^{(j)} \cap B_k^{(j)} = \emptyset$, we have $\mathbb{P}\left( A_k^{(j)} \cup B_k^{(j)} \right) = \mathbb{P}\left( A_k^{(j)} \right)+ \mathbb{P}\left( B_k^{(j)}\right) $. The probability $\mathbb{P}\left( A_k^{(j)} \right)$ is determined by the order statistics of the combined samples $\{x_k^{(1)}, \cdots,x_k^{(N)},y_k^{(j)} \}$. Specifically, $A_k^{(j)}$ occurs iff all $x$-samples are smaller than $y$-samples in coordinate $k$. Let $S \subset S_{N+1}$ be the set of permutations where the first $N$ positions are occupied by $\{x_k^{(1)}, \cdots,x_k^{(N)} \}$ and the last one is $y_k^{(j)}$, one has $\mathbb{P}\left( A_k^{(j)} \right) =\frac{|S|}{|S_{N+1}|} =\frac{N!}{(N+1)!}=\frac{1}{N+1}$. Similarly, we can conclude that $\mathbb{P}\left( B_k^{(j)} \right) = \frac{1}{N+1}$.  With combining Equation \eqref{eq:q1} and \eqref{eq:q}, the overall probability bound is $P \geq \left( 1 - \left( \frac{N-1}{N+1} \right)^d \right)^M$.
\end{proof}

\subsection{Proof of Proposition \ref{propos:2}}
\label{app:pro3}

\begin{proof}
	 By the construction of $W_i$ and assumptions on $\Phi_i$, each $W_i$ satisfies the structural conditions of Lemma \ref{lem:2.2} for the invariant set $\{c_i\}$. First, consider the loss function $	\mathcal{L}_{con}(\theta) = \sum_{i=1}^L \sup_{h \in D_{W_i}} l_{con}^i(h)$, where $l_{con}^i(h) \geq 0$. The condition $\mathcal{L}_{con}(\theta^*) = 0$ holds if and only if for every $i$, $\sup_{h \in D_{W_i}} l_{con}^i(h) = 0$, which implies $l_{con}^i(h) = 0$ for all $h \in D_{W_i}$. By Lemma \ref{lem:2.2}, this is equivalent to $D_{W_i} = \mathcal{D}_{f}(c_i)$ for each $i$ and we denote $D^* = \bigcup_{c_i \in \mathcal{A}} D_{W_i}$. 
	
	To establish set equality, we prove containment in both directions.  For any $h \in \mathcal{D}_{f}(\mathcal{A})$, by definition of the region of attraction, $\lim_{t \to \infty} \| \omega_{\theta_f}(t, h) \|_{\mathcal{A}} = 0$, where $\omega_{\theta_f}$ is a solution to Equation \eqref{eq:2.2}. By continuity of solutions, there exists some $c_i \in \mathcal{A}$ such that $\lim_{t \to \infty} \| \omega_{\theta_f}(t, h) - c_i \| = 0$.Thus, $h \in \mathcal{D}_{f}(c_i) \subseteq D^*$, proving $\mathcal{D}_{f}(\mathcal{A}) \subseteq D^*$. On the other hand, for any $h \in D^*$, there exists $c_i \in \mathcal{A}$ such that $h \in D_{W_i} = \mathcal{D}_f(c_i)$. Since $c_i \in \mathcal{A}$, by definition of regions of attraction, one has $\mathcal{D}_{f}(c_i) \subseteq \mathcal{D}_{f}(\mathcal{A})$. Thus, we have $D^* \subseteq \mathcal{D}_{f}(\mathcal{A})$. The set equality follows from mutual containment, completing the proof.
\end{proof}

\subsection{Proof of Proposition \ref{propos:3}}
\label{app:pro4}

\begin{proof}
	From Proposition \ref{propos:2}, $\mathcal{L}_{con}(\theta^*) = 0$ implies $D_{W_i} = \mathcal{D}_{f}(c_i)$ for each $i$, where $\mathcal{D}_{f}(c_i)$ is the true region of attraction for $c_i$. By construction, each $W_i$ is positive definite near $c_i$ with $\dot{W}_i$ negative definite, ensuring $c_i$ is an equilibrium point and asymptotically stable.
	Assume for contradiction that $D_{W_1} \cap D_{W_2} \neq \emptyset$. Take $h^* \in D_{W_1} \cap D_{W_2}$ and choose $s > 0$ such that $\max\{W_1(h^*), W_2(h^*)\} \leq s < 1$. We define the compact set $\Omega = \{ h \in \mathbb{R}^n : W_i(h) \leq s \ \text{for}\ i=1,2 \}$, which contains $h^*$. Compactness follows from: (1) continuity of $W_i$, (2) boundedness due to $\lim_{\|h\|\to\infty} W_i(h) = 1 > s$, and (3) closedness by construction. Let $\omega(t) = \omega_{\theta_f}(t, h^*)$ be the solution to Equation \eqref{eq:2.2}. For the evolution of \(\omega(t)\), there are two possibilities: either there exists a minimal \(t_1 > 0\) such that \(\omega(t_1) \in \partial\Omega\), or \(\omega(t) \in \Omega\) for all \(t \ge 0\).
	
	\textit{Case 1:} The trajectory $\omega(t)$ first hits the set $\Omega$ at some finite time $t_1 > 0$, i.e., $\omega(t_1) \in \partial \Omega$. At the point $h = \omega(t_1)$, we have $\min \left\{ \dot{W}_1(h), \dot{W}_2(h) \right\} = 0$; otherwise, by continuity, there would exist $t_2 > t_1$ such that $\omega(t_2) \in \Omega$, contradicting the definition of $t_1$ as the first hitting time.
	First, assume both derivatives vanish at $h$, i.e., $\dot{W}_1(h) = 0$ and $\dot{W}_2(h) = 0$. Since the Zubov equation holds ($\mathcal{L}_{con} = 0$), we have:
	\begin{align*}
		\dot{W}_1(h) = 0 &\Rightarrow -\Phi_1(h)(1 - W_1(h)) = 0, \\
		\dot{W}_2(h) = 0 &\Rightarrow -\Phi_2(h)(1 - W_2(h)) = 0.
	\end{align*}
	Given that $W_i(h) < 1$ for $i = 1, 2$, the terms $(1 - W_i(h))$ are positive. Therefore, these equations imply $\Phi_1(h) = 0$ and $\Phi_2(h) = 0$. By the positive definiteness of $\Phi_i$, this forces $h = c_1$ and $h = c_2$, leading to the contradiction $c_1 = c_2$.
	Now, without loss of generality, suppose $\dot{W}_1(h) = 0$ but $\dot{W}_2(h) < 0$. From $\dot{W}_1(h) = 0$ and the Zubov equation, it follows that $\Phi_1(h) = 0$. By positive definiteness, this implies $h = c_1$. However, $\dot{W}_2(h) < 0$ implies $f(h; \theta_f) \neq 0$, which contradicts the fact that $c_1$ is an equilibrium point.

	\textit{Case 2:} $\omega(t) \in \Omega$ for all $t \geq 0$. By LaSalle's invariance principle \cite{Khalil:1173048}, since $\dot{W}_i(\omega(t)) \leq 0$ and $\Omega$ compact, $\omega(t)$ approaches the largest invariant set where $\dot{W}_1 = \dot{W}_2 = 0$. From the Zubov equation $ \dot{W}_i = -\Phi_i(h)(1-W_i(h))$ and the positive definiteness of $\Phi_i$, we have that $\dot{W}_i(h) = 0$ if and only if $h = c_i$ (since $1-W_i(h) > 0$ for $h \ne c_i$), which is contradiction the fact that $c_1 \ne c_2$. Therefore, the intersection must be empty.
\end{proof}

\subsection{Proof of Proposition \ref{propos:4}}
\label{app:pro5}

\begin{proof}
As established earlier, we can easily obtain that $l_{con}^i(h)$ is zero for any $h \in D_{W_i}$ and $i \in \left\{ {1, \cdots ,L} \right\}$ since \( \mathcal{L}_{con}(\theta^*) = 0 \).  Thus, for the input data $(x,y)$ one has ${\dot W_y}(h) =  - {\Phi _y}(h)(1 - {W_y}(h))$ for all  $h \in D_{W_y}$. To further illustrate the conclusion,  assume to the contrary that  $\left\{ {{h_x}\left( t \right):t \in \left[ {0,T} \right]} \right\}  \not\subset  D_{W_y} $, i.e., there exists $t_1\in (0,T)$ satisfying $h_x(t_1) \in \partial D_{W_y}$ and $h_x(t) \in  D_{W_y}$ for  $t_1< t \le T$. Letting $K(h) = 1- W_y(h)$, we have $\dot{K}(t) = \Phi_y(h_x(t))K(t)$. 

By taking the integration of the above equation on $[t,T]$ with $t_1<t<T$, one has $ \int_{{K(h_x(t))}}^{{K(h_x(T))}} {\frac{{dK}}{K}} = \int_t^{\rm{T}} {{\Phi _y}({h_x}(s))} {\mkern 1mu} ds$ and so $K(h_x(T))  = K(h_x(t)){e^{\int_t^T {{\Phi _y}({h_x}(s))} {\kern 1pt} ds}}$, which shows that 
\begin{equation*}
	1 - {W_y}(h_x(t))  = [1 - {W_y}(h_x(T))]{e^{ - \int_{{t}}^T \Phi_y(h_x(s)){\kern 1pt} ds}}.
\end{equation*}
Letting $t \to t_1$ leads to $ \mathop {\lim }\limits_{t \to {t_1}}1 - {W_y}(h_x(t)) = 0$ and $ \mathop {\lim }\limits_{t \to {t_1}} (1 - {W_y}(h_x(T))){e^{ - \int_{{t}}^T  \Phi_y(h_x(s)) {\kern 1pt} ds}} > 0$, which is a contradiction. Thus, if $h_x(T)\in D_{W_y}$, we have $h_x(t) \in D_{W_y}$ for any $t\in [0,T)$.
\end{proof}

\subsection{Proof of Proposition \ref{prop:robustness}}
\label{app:pro6}
\begin{proof}
	By Proposition \ref{propos:4}, we have $W_y(h_x(0)) < 1$ so that $\varepsilon^* > 0$. The distance between the clean and perturbed features is bounded by the Lipschitz continuity of $\phi$
	\begin{equation*}
		\|h_x(0) - h_{x'}(0)\|  \leq L_\phi \|\delta\|.
	\end{equation*}
	Similarly, the change in the Lyapunov value is bounded by the Lipschitz continuity of $W_y$:
	\begin{equation*}
		\| W_y(h_x(0)) - W_y(h_{x'}(0)) \| \leq L_W L_\phi \|\delta\|.
	\end{equation*}
	Therefore, we have an upper bound for $W_y(h_{x'}(0))$:
	\begin{align*}
		W_y(h_{x'}(0)) &\leq W_y(h_x(0)) + L_W L_\phi \|\delta\| \\
		&< W_y(h_x(0)) + L_W L_\phi \varepsilon^* = W_y(h_x(0)) + (1 - W_y(h_x(0))) = 1.
	\end{align*}
	The strict inequality $W_y(h_{x'}(0)) < 1$ implies $h_{x'}(0) \in D_{W_y}$. By Proposition \ref{propos:2}, starting from $h_{x'}(0) \in D_{W_y}$, the entire trajectory remains in $D_{W_y}$ and converges to the equilibrium $c_y$. Thus, $h_{x'}(T) \in D_{W_y}$. 
\end{proof}

\section{Additional Experiments}
\label{app:exp}
\begin{figure*}[ht] 
	\centering
	\includegraphics[width=\textwidth]{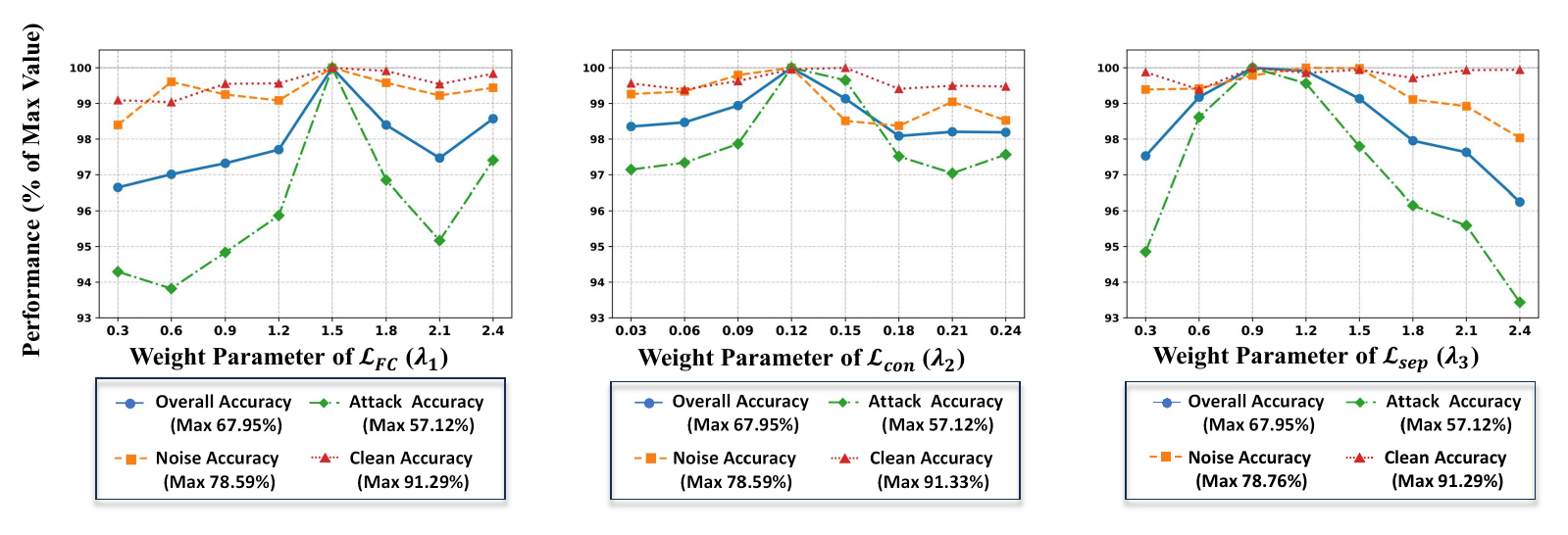} 
	\caption{Performance comparison of different loss weight configurations normalized to maximum values. The x-axis represents eight weight configurations, while the y-axis shows relative performance percentage. Four metrics are evaluated: overall accuracy (blue solid line), noise robustness (orange dashed line), adversarial robustness (green dash-dot line), and clean accuracy (red dotted line). Performance is normalized to the maximum value achieved by any configuration for fair comparison.} 
	\label{fig:crue}
\end{figure*}

\noindent \textbf{Loss weight analysis.} We further analyze the impact of loss weights $\lambda_1$, $\lambda_2$, and $\lambda_3$ on CIFAR-10, as visualized in Figure \ref{fig:crue}. Despite clean accuracy is the most stable across all weight variations, it shows its strongest correlation with $\lambda_1$ (the weight of the auxiliary cross-entropy loss $\mathcal{L}_{FC}$). This suggest that maintaining a strong supervisory signal from the standard classification head is beneficial for pure accuracy. For adversarial robustness, the all weight are crucial for achieving high robustness against adversarial attacks. Notably, $\lambda_3$ appears to offer a slightly wider optimal range, suggesting that actively enlarging the margin between basins is effective for adversarial defense once approximate consistency is enforced. Robustness to stochastic corruptions is significantly influenced by both $\lambda_2$ and $\lambda_3$. These findings validate the necessity of all three components in our tripartite loss and provide practical guidance for balancing them.

\begin{table*}[ht]
	\centering
	\caption{Ablation study of ODE solver configurations and integration time $T$ with corresponding discretization steps $\Gamma$ on CIFAR-10 ($\varepsilon=8/255$). ``Original'' denotes clean accuracy. The ``Avg'' column reports the average robustness accuracy over the ten corruption and attack types (excluding clean accuracy). Unit: \%.}
	\label{table:ab2}
	\begin{tabular}{lcccccccccccc}
		\toprule
		\multicolumn{1}{c}{\textbf{Configuration}} & \textbf{Orig} & \textbf{Gaussian} & \textbf{Glass} & \textbf{Shot} & \textbf{Speckle} & \textbf{Motion} & \textbf{FGSM} & \textbf{BIM} & \textbf{PGD} & \textbf{APGD} & \textbf{Jitter} & \textbf{Avg} \\\midrule
		Dopri5 (rho=0.001)                   & 91.31            & 71.53             & 76.09          & 66.82         & 78.84            & 83.18           & 58.12         & 60.97        & 75.25        & 58.42         & 44.92           & 67.41       \\
		Dopri5 (rho=0.1)                     & 91.13            & 70.02             & 76.57          & 66.10         & 78.64           & 82.61           & 55.67         & 58.52        & 74.90        & 56.77         & 43.62           & 66.34        \\
		RK4 (stepsize=0.1 )                 & 91.08            & 69.70             & 76.22          & 65.31         & 78.26           & 82.36           & 52.52         & 58.32        & 74.02        & 56.62         & 43.40           & 65.67        \\
		Eulr (stepsize=0.1)                 & 90.83            & 69.11             & 75.71          & 64.62         & 77.72            & 81.89           & 52.97         & 57.89        & 73.82        & 56.21         & 43.05           & 65.30     \\ \midrule 
		
		Ours (T=1, $\Gamma$=10) & 91.29 & 71.06 & 76.43 & 66.71 & 78.80 & 82.86 & 56.83 & 59.94 & 75.01 & 58.03 & 44.87 & 67.05\\ 	\midrule
		$T=2,\Gamma=15$  & 91.12 & 71.06 & 75.40 & 66.83 & 78.69 & 82.35 & 55.34 & 58.99 & 74.17 & 56.75 & 43.37 & 66.30 \\
		$T=5,\Gamma=30$  & 91.04 & 69.25 & 74.70 & 65.97 & 77.29 & 81.38 & 54.15 & 57.16 & 73.24 & 56.25 & 43.16 & 65.26 \\
		$T=7,\Gamma=40$  & 90.57 & 72.20 & 76.23 & 67.35 & 78.34 & 82.04 & 55.52 & 59.16 & 74.55 & 57.10 & 45.19 & 66.77 \\
		$T=10,\Gamma=60$ & 90.05 & 72.49 & 76.90 & 68.48 & 79.61 & 82.95 & 56.92 & 60.20 & 75.73 & 58.11 & 45.38 & 67.68 \\  
		\bottomrule    
	\end{tabular}
\end{table*}

\noindent \textbf{Solver configuration analysis.} We evaluate the impact of ODE solver on Zubov-Net using CIFAR-10. Four configurations are compared: Dopri5 with tolerances 0.001 and 0.1, RK4 with step size 0.1, and Euler with step size 0.1. As solver accuracy decreases, both clean accuracy and robustness degrade monotonically. As shown in Table \ref{table:ab2}, the highest-precision setting (Dopri5, rho=0.001) achieves the best average robustness accuracy (67.41\%), outperforming the lowest-precision Euler by 2.01\%. Our standard configuration uses Dopri5 with rho=0.1 during training and rho=0.001 during testing, which balances computational efficiency and robustness.

\noindent \textbf{Time and discretization analysis.} We further evaluate the impact of integration time $T$ and discretization resolution $\Lambda$ (denoted as $\Gamma$ in Table \ref{table:ab2}).  Our standard configuration ($T=1$, $\Gamma=10$) achieves a clean accuracy of 91.29\% and an average robustness accuracy of 67.05\%.  When $T$ and $\Gamma$ increase from $T=1$ to $T=10$, the clean accuracy decreases from 91.29\% to 90.05\%. However, the average robustness accuracy varies non-monotonically: it decreases from 67.05\% at $T=1$ to 66.77\% at $T=5$, then rises to 67.68\% at $T=10$. Among these settings, $T=10$ and $\Gamma=60$ yield the highest average robustness accuracy, though with a 1.24\% drop in clean accuracy relative to our standard configuration. Our configuration maintains competitive robustness with minimal clean accuracy degradation, and requires less computational cost than configurations with larger $T$ and $\Gamma$.

\begin{figure*}[!t] 
	\centering
	\includegraphics[width=\textwidth]{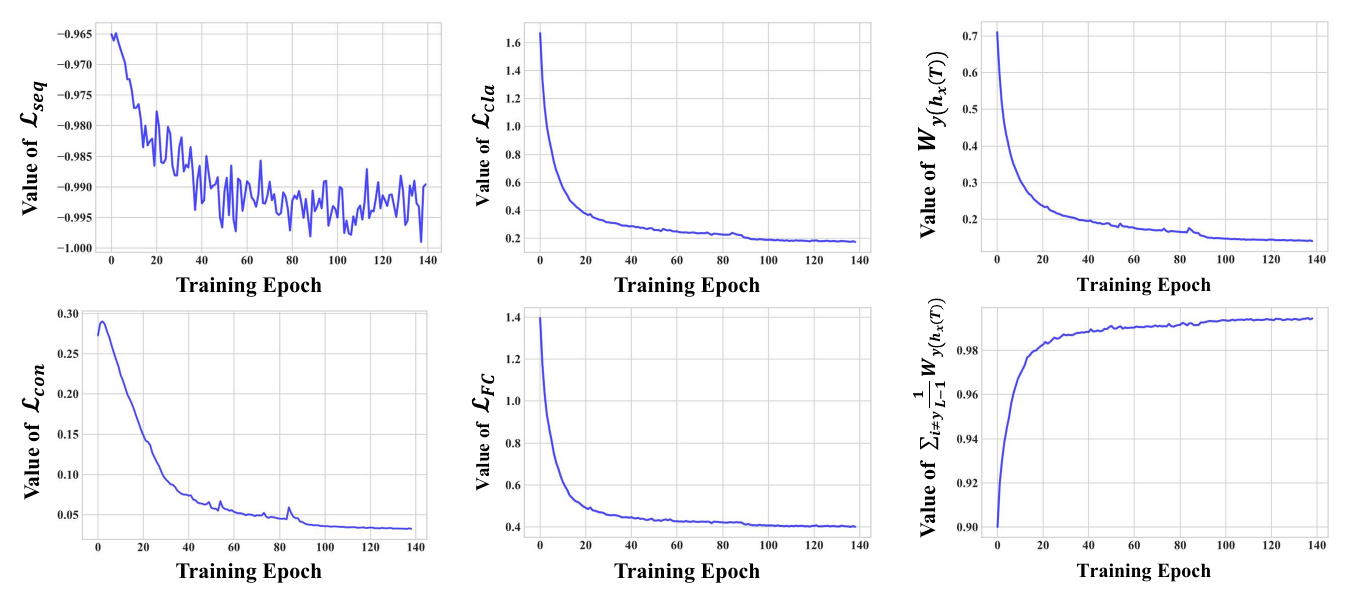} 
	\caption{Training dynamics of loss components and Lyapunov function values on CIFAR-10. It contains that separation loss $\mathcal{L}_{sep}$, classification loss $\mathcal{L}_{cla}$, consistency loss $\mathcal{L}_{con}$, auxiliary cross-entropy loss $\mathcal{L}_{FC}$, the target class Lyapunov function value, and the average of Lyapunov function values for non-target classes.  } 
	\label{fig:training_dynamics}
\end{figure*}
\noindent \textbf{Training dynamics analysis.} As shown in Figure \ref{fig:training_dynamics}, the classification loss of Lyapunov classifier $\mathcal{L}_{cla}$ converges steadily, which shows that he prescribed regions of attraction (PRoAs) are progressively aligned with the discriminative structure of data distribution. Moreover, the rapid minimization of the separation loss $\mathcal{L}_{sep}$, alongside the opposing trends of the target- and non-target-class Lyapunov function values, provides direct evidence of the active basin separation and geometric shaping during training. The steady decrease in the consistency loss $\mathcal{L}_{con}$ around a small value $0.03$ demonstrates the progressive satisfaction of the Zubov-driven matching condition. This close alignment between PRoAs and RoAs, despite the non-zero residual, is sufficient to drive the robustness improvements observed in our experiments. Together, these dynamics show that the tripartite loss enables simultaneous optimization for classification, attraction basin separation, and approximate satisfaction of the Zubov stability condition.

\end{document}